\relax
\documentclass[letterpaper]{article}
\usepackage{ForArxiv} 
\usepackage{times}  
\usepackage{helvet} 
\usepackage{courier}
\usepackage[hyphens]{url}  
\usepackage{graphicx} 
\usepackage{natbib}   
\usepackage{caption}  
\usepackage[caption=false]{subfig}

\title{Constraining Logits by Bounded Function
for Adversarial Robustness}
\author{
Sekitoshi Kanai,\textsuperscript{\rm 1}
Masanori Yamada,\textsuperscript{\rm 2}
Shin'ya Yamaguchi,\textsuperscript{\rm 1}\\
Hiroshi Takahashi,\textsuperscript{\rm 1, 3}
Yasutoshi Ida,\textsuperscript{\rm 1, 3}\\
}
\affiliations{
\textsuperscript{\rm 1}NTT Software Innovation Center,  
\textsuperscript{\rm 2}NTT Secure Platform Laboratories,
 \textsuperscript{\rm 3}Kyoto University\\
 sekitoshi.kanai.fu@hco.ntt.co.jp, masanori.yamada.cm@hco.ntt.co.jp, shinya.yamaguchi.mw@hco.ntt.co.jp, hiroshi.takahashi.bm@hco.ntt.co.jp, yasutoshi.ida.yc@hco.ntt.co.jp
}

\usepackage{amsthm}

\newtheorem{definition}{Definition}

\newtheorem{corollary}{Corollary}

\newtheorem{theorem}{Theorem}

\newtheorem{proposition}{Proposition}

\newtheorem*{assumption*}{Assumption}
\usepackage{bm}

\usepackage{nccmath}

\usepackage{amsmath,cases}
\usepackage{booktabs}

\newcommand{\req}[1]{eq.~(\ref{#1})}
\newcommand{\rfig}[1]{Fig.~\ref{#1}}
\newcommand{\rtab}[1]{Tab.~\ref{#1}}

\usepackage{url}

\begin{document}
\maketitle

\begin{abstract}
    We propose a method for improving adversarial robustness by
    addition of a new bounded function just before softmax. 
    Recent studies hypothesize
    that small logits (inputs of softmax) by logit regularization can improve adversarial robustness of deep
    learning.
    Following this hypothesis,
    we analyze norms of logit vectors at
    the optimal point under the assumption of universal approximation and 
    explore new methods for constraining logits by addition of a bounded function before softmax.
    We theoretically and empirically reveal that small logits by addition of a common activation function,
    e.g., hyperbolic tangent, do not improve adversarial robustness since input vectors of the function 
    (pre-logit vectors) can have large norms. 
    From the theoretical findings, we develop the new bounded function. 
    The addition of our function improves adversarial robustness
    because it makes logit and pre-logit vectors have small norms.
    Since our method only adds one activation function before softmax, 
    it is easy to combine our method with adversarial training.
    Our experiments demonstrate that our method is comparable to logit regularization
    methods in terms of accuracies
    on adversarially perturbed datasets without adversarial training. 
    Furthermore, it is superior or comparable to
    logit regularization methods and a recent defense method (TRADES)
    when using adversarial training.
\end{abstract}
\section{Introduction}
Deep neural networks (DNNs) are used
in many applications, e.g., image recognition \cite{resnet} and 
machine translation \cite{vaswani2017attention}, and have achieved great success.
Although DNNs can handle data accurately, 
they are vulnerable to adversarial examples, which are imperceptibly
perturbed data to make DNNs misclassify data \cite{szegedy2013intriguing}.
To investigate vulnerabilities of DNNs and improve adversarial robustness, many adversarial attack and defense methods 
have been presented and evaluated \cite{carmon2019unlabeled,pgd,pgd2,papernot2017practical,Fp,zhang2019you,ross2018improving,absum,mlloo,fisherdetct,resistAA}.

Among defense methods, adversarial training is regarded as a promising method
\cite{pgd,pgd2}.
Adversarial training generates adversarial examples of training data and
trains DNNs on these adversarial examples. 
On the other hand, since the generation of adversarial examples requires high computation cost,
several studies introduced defense methods not using adversarial examples \cite{LogitSq,MMC,warde201611}.
\citet{warde201611} 
showed that label smoothing can be used as the efficient defense methods.
Similarly, \citet{LogitSq} presented logit squeezing that imposes a penalty of norms of
logit vectors, which are input vectors of softmax.
Since recent studies have shown that label smoothing also induces small logits, they
regard label smoothing and logit squeezing as logit regularization methods and
experimentally investigate the relation between robustness and logit norms \cite{mosbach2018logit,shafahi2019,shafahibatch,summers2019improved}.
\citet{shafahi2019} showed that adversarial training actually
reduces the norms of logits along with the increase in the strength of the attacks in training.
According to these studies, constraining of logit norms to be small values is one approach
to improve adversarial examples.

In this paper, we propose a method of constraining the logit norms that uses
 a bounded activation function just before softmax
on the basis of our theoretical findings about logit regularization.
To understand the effect of logit regularization, 
we analyze the optimal logits for training of a neural network that has universal approximation.
In this analysis, we first prove that (i) norms of the optimal logit vectors of cross-entropy are infinitely
large values, and
(ii) logit squeezing and label smoothing make norms of the optimal logit vectors be finite values.
Infinitely large logits mean that the function from data points to 
logits does not have finite Lipschitz constants of which constraints are used
for adversarial robustness \cite{parseval,farnia2018generalizable,NIPS2019_9319,lmt}.
Next, to verify the effect of small logit norms, we evaluate robustness of models the logit vectors of
 which have various norms. 
From the experiments, we confirm that
models can be robust against adversarial attacks if norms of logits are below a certain value.
This observation suggests that the addition of a bounded function 
just before softmax can improve adversarial robustness.
However, we also reveal that adding common bounded activation functions, e.g., hyperbolic tangent, does not improve adversarial robustness.
This is because these functions are monotonically increasing,
and the optimal inputs (hereinafter, we call pre-logits) of these functions become infinitely large values;
i.e., models from data points to pre-logits do not have finite Lipschitz constants.
To overcome this drawback, we develop a new function called bounded logit function (BLF).
BLF is bounded by finite values, and pre-logits at the maximum and minimum points
are also finite values.
As a result, the addition of BLF just before softmax makes the optimal logits and pre-logits have finite values.
Since our method only adds one activation function before softmax, 
it is easy to combine BLF with adversarial training.
We experimentally confirmed that 
BLF can improve robustness more than logit squeezing without using adversarial examples.
Moreover, BLF with adversarial training is superior or comparable to
adversarial training with logit squeezing, label smoothing, and TRADES, which is a recent 
strong defense method using adversarial examples in training \citep{TRADES}, 
in terms of accuracy on MNIST and CIFAR10 attacked by gradient-based white-box attack (PGD \cite{pgd,pgd2}) and gradient-free attacks (SPSA \cite{SPSA} and Square Attack \cite{SQUARE}).
\section{Preliminaries}
\label{preSec}
\subsection{\!\!\!Softmax cross-entropy and logit regularization}
To classify the $i$-th data point $\bm{x}^{(i)}\!\in\!\mathcal{X}$ where $\mathcal{X}$ is a data space,
neural networks learn a probability distribution over $M$ classes
conditioned on the data point $\bm{x}$ as $P_{\bm{\theta}}(k|\bm{x}^{(i)})$ for $k\!=\!1,\dots,M$
where $\bm{\theta}$ is a parameter vector.
Let $\bm{z}_{\bm{\theta}}(\bm{x}^{(i)})=[z_{\bm{\theta},1}(\bm{x}^{(i)}), z_{\bm{\theta},2}(\bm{x}^{(i)}),\dots,z_{\bm{\theta},M}(\bm{x}^{(i)})]^T$ be a logit vector composed of a logit $z_{\bm{\theta},k}(\bm{x}^{(i)})$, which is an input vector
of softmax $\bm{f}_s(\cdot)$, corresponding to the data point $\bm{x}^{(i)}$.
The $k$-th element of softmax $[\bm{f}_s(\cdot)]_k$ represents
the conditional probability of the $k$-th class
\begin{align}
\textstyle
P_{\bm{\theta}}(k|\bm{x}^{(i)})\!=\!
\left[\bm{f}_s(\bm{z}_{\bm{\theta}}(\bm{x}^{(i)}))\right]_k
\!=\!\frac{\mathrm{exp}(z_{\bm{\theta},k}(\bm{x}^{(i)}))}
{\sum_{m=1}^{M} \mathrm{exp}(z_{\bm{\theta},m}(\bm{x}^{(i)}))}.
\label{sbeq}
\end{align}
To train the model $\bm{f}_s(\bm{z}_{\bm{\theta}}(\bm{x}))$, cross-entropy loss $\mathcal{L}_{CE}$ is
 used as loss functions in classification tasks.
Cross-entropy loss $\mathcal{L}_{\mathrm{CE}}(\bm{z}_{\bm{\theta}}(\bm{x}^{(i)}),\bm{p}^{(i)})$ is
\begin{align}
        \textstyle
        \!\!\!\mathcal{L}_{\mathrm{CE}}(\bm{z}_{\bm{\theta}}(\bm{x}^{(i)}),\bm{p}^{(i)})
        \!=\!-\sum_{k=1}^M p_k^{(i)}\mathrm{log} [\bm{f}_s(\bm{z}_{\bm{\theta}}(\bm{x}^{(i)}))]_k,                
\label{ce}
\end{align}
where $\bm{p}^{(i)}$ is a target vector for $\bm{x}^{(i)}$.
Target vectors are generally one-hot vectors as $p_k\!=\!1$ for $k\!=\!t$ and $p_k=0$ for $k\!\neq\!t$
when the label of $\bm{x}^{(i)}$ is $t$.
For a training dataset $\{(\bm{x}^{(i)},\bm{p}^{(i)})\}_{i=1}^N$,
an objective function $J$ of training is
$J=\frac{1}{N}\sum_{i=1}^N\mathcal{L}_{\mathrm{CE}}(\bm{z}_{\bm{\theta}}(\bm{x}^{(i)}),\bm{p}^{(i)})$.

Since softmax cross-entropy, which is a combination of softmax and cross-entropy, for one-hot vectors can make models over-confident,
label smoothing assigns small probabilities to the target values as
$p_t=
1-\alpha$ and 
$p_k=\frac{\alpha}{M-1}$ for $k\neq t$
where $\alpha$ is a hyper-parameter \cite{LabelSmo}.
Label smoothing can alleviate over-confidence and improve generalization
performance \cite{LabelSmo}.
Recent studies have shown that label smoothing can be used as a defense method
against adversarial attacks \cite{shafahi2019,shafahibatch,summers2019improved,warde201611}.

\citet{LogitSq} presented logit squeezing as another way for improving adversarial robustness. 
The objective function of logit squeezing is
\begin{align}
\textstyle
J=\frac{1}{N}\sum_{i=1}^N\left(\mathcal{L}_{\mathrm{CE}}(\bm{x}^{(i)},\bm{p}^{(i)})+\frac{\lambda}{2}||\bm{z}_{\bm{\theta}}(\bm{x}^{(i)})||_2^2\right),
\end{align}
where $\lambda\!>\!0$ is a regularization constant.
This objective function keeps logit norms having small values, and it is
experimentally confirmed that logit squeezing can improve the adversarial robustness.
\subsection{Related work}
Since many studies have been conducted on adversarial attacks and defenses,
we mainly review studies that handle defense methods based on logits.
First, we discuss defense methods using adversarial examples.
\citet{LogitSq} presented adversarial logit pairing, which makes logits of adversarial examples
be similar to logits of clean examples by an additional regularization term 
penalizing the $L_2$ norm distance between them.
Similarly to adversarial logit pairing,
TRADES \cite{TRADES} uses a regularization term that evaluates the difference between softmax
outputs of neural networks for clean examples and for adversarial examples instead of logits.
Since TRADES outperforms adversarial logit pairing, 
we compared our method with TRADES rather than adversarial logit pairing \cite{TRADES}.

Next, we discuss defense methods without using adversarial examples. 
As mentioned above, \citet{LogitSq} also presented logit squeezing inspired by logit pairing, and 
several studies were conducted on the effectiveness of logit squeezing 
\cite{shafahi2019,shafahibatch,summers2019improved}.
\citet{shafahi2019} investigated the effects of logit squeezing and label smoothing
and presented a combination of logit regularization (label smoothing or
logit squeezing) and random Gaussian noise. 
In our study, we evaluated the combination of each logit regularization method and adversarial
training instead of Gaussian noise
 because adversarial perturbations are the worst noise for models, 
i.e., training with them can improve robustness more than that with Gaussian noise.
\citet{summers2019improved} experimentally showed that label smoothing also makes logits have a small range like logit squeezing.
They evaluated a crafted logit regularization method, which combines label smoothing, mixup,
and logit pairing. 
Our method can be one component of the crafted logit regularization method because our method is simple and
only adds an activation function before softmax.

\citet{mosbach2018logit} investigated the loss surface of logit squeezing and showed that
logit squeezing seems to just mask or obfuscate gradient \cite{best} rather than improving  robustness.
However, logit squeezing can still slightly improve robustness against gradient-free attacks,
which does not depend on a gradient. 
Note that the effectiveness of the combination logit squeezing and adversarial training
is still unclear. 
In Section~\ref{exp}, we evaluate logit regularization methods by using 
gradient-based and two gradient-free attacks, and reveal that logit regularization can
improve robustness of adversarial training.
In addition, we investigate the limitation of logit regularization methods 
by using various attacks, e.g., targeted attacks, in appendix~\ref{LimSec}.
\section{Analysis of logit regularization methods for adversarial robustness}
\label{logitSec}
In this section, we first investigate logits
obtained by softmax cross-entropy and logit regularization methods.
Next, we experimentally show the relation between logit norms
and robustness.
All the proofs are provided in appendix~\ref{ProfSsec}.
\subsection{\!\!\!Optimal logits for minimization of training loss}
\label{AsSec}
To clarify the background of our results, we first show the assumption of our analysis inspired 
by the universal approximation properties of neural networks.
\begin{assumption*}
We assume that (a) if data points have the same values as $\bm{x}^{(i)}\!=\!\bm{x}^{(j)}$, they have 
the same labels as $\bm{p}^{(i)}\!=\!\bm{p}^{(j)}$,
(b) the logit vector $\bm{z}_{\bm{\theta}}(\bm{x})$ can be an arbitrary vector for each data point, and
(c) the optimal point $\bm{\theta}^*\!=\!\mathrm{arg}\!\min_{\bm{\theta}}\!\frac{1}{N}\!\sum_{i=1}^{N}\! \mathcal{L}(\bm{z}_{\bm{\theta}}(\bm{x}^{(i)}),\bm{p}^{(i)})$ achieves $\mathcal{L}(\bm{z}_{\bm{\theta}^*}(\bm{x}^{(i)}),\bm{p}^{(i)})\!=\min_{\bm{\theta}}\mathcal{L}(\bm{z}_{\bm{\theta}}(\bm{x}^{(i)}),\bm{p}^{(i)})$ for all $i$.
\end{assumption*}
This assumption ignores generalization performance and takes into account deterministic labels.
This assumption is satisfied if 
we use the  models that can be arbitrary functions and
minimize softmax cross-entropy on the dataset where the
same data points have the same labels.
Though it is a strong assumption,\footnote{Since several papers have shown that over-parameterized network with least squares loss can achieve zero training loss \cite{du2019gradient,du2018gradient},
it might not be a very strong assumption.}
our analysis is valuable for understanding the behavior of the logits 
since DNNs have large representation capacity and we assign a label for each data point
with no duplication. 
 
From the assumption, we can regard the optimal logits $\bm{z}_{\bm{\theta}^*}
(\bm{x}^{(i)})\!=\!\mathrm{arg}\!\min_{\bm{z}_{\bm{\theta}}(\bm{x}^{(i)})}\!\mathcal{L}(\bm{z}_{\bm{\theta}}(\bm{x}^{(i)}),\bm{p}
^{(i)})$ for the $i$-th data point 
as the logits obtained by minimization of the training objective functions.
Next, we show the property of the optimal logits for softmax cross-entropy.
\begin{theorem}
\label{SCETh}
If we use softmax cross-entropy and one-hot vectors as target values, 
at least one element of the optimal logits $\bm{z}_{\bm{\theta}^*}(\bm{x}^{(i)})\!=\!\mathrm{arg}\min_{\bm{z}_{\bm{\theta}}} \mathcal{L}
_{\mathrm{CE}}(\bm{z}_{\bm{\theta}}(\bm{x}^{(i)}),\bm{p}^{(i)}) $
does not have a finite value.
\end{theorem}
This theorem indicates that softmax cross-entropy enlarges logit norms. 
Though this result has been mentioned in several studies \cite{LabelSmo,warde201611},
we show this theorem to clarify our motivation and claims.
From this theorem, we can derive the following corollary:
\begin{corollary}
\label{SCECoro}
If all elements of inputs $x_k^{(i)}$ are normalized as $0\leq x_k^{(i)}\leq 1$ and training dataset has at least two different labels,
the optimal logit function $\bm{z}_{\bm{\theta}^*}(\bm{x})$
for softmax cross-entropy is not globally Lipschitz continuous function, i.e.,
 there is not a finite constant $C\geq 0$ as
 \begin{align}\textstyle
 ||\bm{z}_{\bm{\theta}^*}(\bm{x}^{(i)})-\bm{z}_{\bm{\theta}^*}(\bm{x}^{(j)})||_{\infty}
 &\leq C ||\bm{x}^{(i)}-\bm{x}^{(j)}||_{\infty},\\
 \forall\bm{x}^{(i)},\bm{x}^{(j)}&\in \mathcal{X}.\nonumber
 \end{align}
\end{corollary}
This corollary indicates that the optimal logit function for softmax cross-entropy does not 
have a Lipschitz constant; logit vectors can be drastically changed
by small perturbations. 
This fact does not immediately mean that the models are vulnerable since
the logit gaps between the correct label and other labels on given data
are also infinite in this case.
Even so, 
since adversarial examples are outside of the training data,
it is difficult to expect the outputs for adversarial examples.
In fact, constraints of Lipschitz constants are used for improving adversarial robustness \cite{parseval,farnia2018generalizable,NIPS2019_9319,lmt}.
This corollary also indicates that finite logit values are necessary conditions
for a Lipschitz constant.
Thus, robust models have finite logit values, which is in agreement with the empirical observation that adversarial training
reduces the norms of logits along with robustness \cite{shafahi2019}.

Next, we consider the optimal logits when using logit regularization methods.
In the same way as Theorem~\ref{SCETh}, we can show the following propositions:
\begin{proposition}
    The optimal logits for label smoothing $\bm{z}_{\bm{\theta}^*}(\bm{x}^{(i)})\!=\!\mathrm{arg}\min_{\bm{z}_{\bm{\theta}}}\mathcal{L}
    _{\mathrm{CE}}(\bm{z}_{\bm{\theta}}(\bm{x}^{(i)}),\bm{p}^{(i)})$
    satisfy
    \begin{align*}
    \textstyle
    z_{\bm{\theta}^*,k}(\bm{x}^{(i)})\!=\!
    \begin{cases}
        \log(\frac{1\!-\!\alpha}{\alpha}\!\sum_{m\neq t}\!\mathrm{exp}((z_{\bm{\theta}^*,m}(\bm{x}^{(i)}))))&\!\!\!\!\!k=t
        \\
        \log(\frac{\alpha}{M\!-\!1\!-\!\alpha}\!\sum_{m\neq k}\mathrm{exp}(z_{\bm{\theta}^*,m}(\bm{x}^{(i)})))&\!\!\!\!\!k\neq t.
    \end{cases}
    \end{align*}
    If an element of $\mathrm{exp}(\bm{z_{\bm{\theta}^*}}(\bm{x}^{(i)}))$ has a finite
    value, all elements of $\bm{z}_{\bm{\theta}^*}(\bm{x}^{(i)})$ have finite
    values.
\end{proposition}
\begin{proposition}
The optimal logits for logit squeezing
$\!\bm{z}_{\bm{\theta}^*}(\bm{x}^{(i)})\!=\!\mathrm{arg}\!\min_{\bm{z}_{\bm{\theta}}} \mathcal{L}
_{\mathrm{CE}}(\bm{z}_{\bm{\theta}}(\bm{x}^{(i)}),\bm{p}^{(i)}\!)\!+\!\frac{\lambda}{2}\!||\bm{z}_{\bm{\theta}}(\bm{x}^{(i)})||_2$ 
satisfy
\begin{align*}
\!z_{\bm{\theta}^*\!,k}(\bm{x}^{(i)})\!=\!\begin{cases}\textstyle
(-[f_s(\bm{z}_{\bm{\theta}^*}(\bm{x}^{(i)})\!)]_t+1)/\lambda &\textstyle k=t\\ \textstyle
\!=\!-[f_s(\bm{z}_{\bm{\theta}^*}(\bm{x}^{(i)})\!)]_k/\lambda&\textstyle k\neq t.\end{cases}
\end{align*}
Namely, all elements of the optimal logit vector $\bm{z}_{\bm{\theta}^*}(\bm{x}^{(i)})$ have finite values.
\end{proposition}
These propositions indicate that logit regularization methods enable
the optimal logit values to have finite values; i.e., these methods satisfy
the necessary condition of Lipschitz continuous.
This property might improve robustness since the logit functions 
do not tend to change drastically by small perturbation on input data points.
If logit regularization methods induce small Lipschitz constants,
the models can be robust against adversarial examples. 

From the hypothesis that small logits can improve adversarial robustness, 
we consider approaches to bound logits other than logit regularization methods.
As an alternative to logit regularization, we can constrain logits by addition of a
bounded activation function just before softmax.
As such functions, hyperbolic tangent (tanh) and sigmoid functions
are common functions in neural networks.
If we use these monotonically increasing functions just before softmax,
we have the following theorem: 
\begin{theorem}
\label{mono}
Let $g(z)$ be tanh or sigmoid function and $\gamma$ be a hyper-parameter
 satisfying $0\!<\!\gamma\!<\!\infty$.
If we use tanh or sigmoid function before softmax
 as $\bm{f}_s(\gamma g(\bm{z}_{\bm{\theta}}(\bm{x}^{(i)})))$, 
all elements of the optimal pre-logit vector
 $\bm{z}_{\bm{\theta}^*}(\bm{x}^{(i)})\!=\!\mathrm{arg}\min_{\bm{z}_{\bm{\theta}}} \mathcal{L}
_{\mathrm{CE}}(\gamma g(\bm{z}_{\bm{\theta}}(\bm{x}^{(i)})),\bm{p}^{(i)}) $
do not have finite values
while all elements of the optimal logit vector $\gamma g(\bm{z}_{\bm{\theta}^*})$ has finite values.
\end{theorem}
This theorem indicates that though optimal logits become finite values by addition of tanh or sigmoid,
the pre-logit functions $\bm{z}_{\bm{\theta}^*}(\bm{x})$ do not have finite Lipschitz constants.
Therefore, the pre-logit can be changed by small perturbation.
Thus, this theorem indicates that bounded logits might not be sufficient for adversarial
robustness.
\captionsetup[subfloat]{farskip=0.5pt,captionskip=2pt}
\begin{figure}[tbp]
\centering
\includegraphics[width=\linewidth]{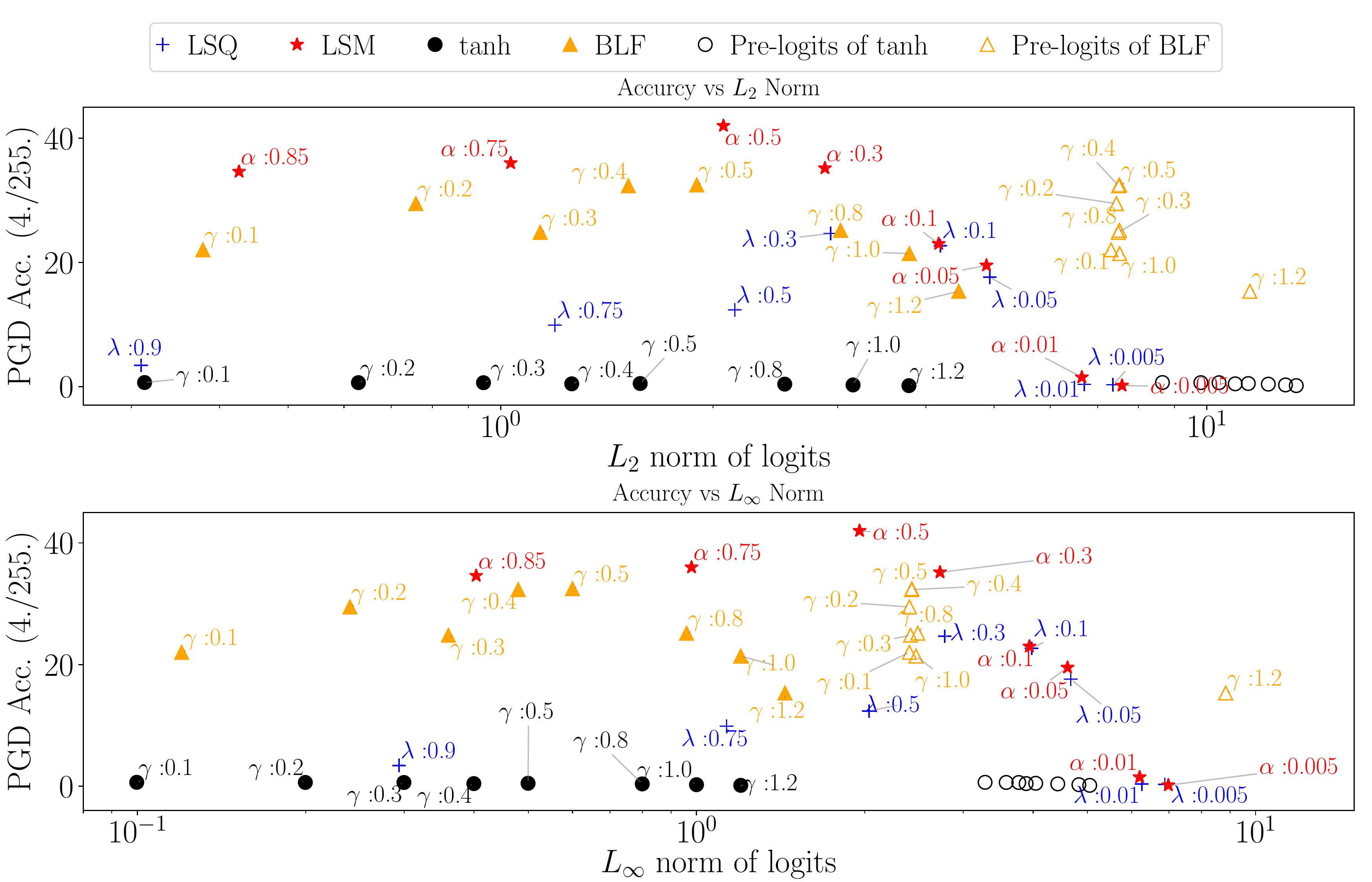}
\caption{$L_2$ (top) and $L_\infty$ (bottom) norms of $\bm{z}$ or $\gamma g(\bm{z})$ vs robust accuracies. Pre-logits of BLF and tanh denote accuracies against norms of $\bm{z}$ instead of norms of $\gamma g(\bm{z})$.~We omit $\gamma$ on Inputs of tanh to increase visibility.}
\label{LogitVSRob}
\end{figure}
\subsection{Empirical evaluation of logit regularization}
As shown above, logit regularization can induce the finite logit values,
and tanh and sigmoid functions cannot keep pre-logits small.
In this section, we empirically investigate the relation between logit norms and
 adversarial robustness.
We evaluated the average norms of logits on clean data of CIFAR10 
and accuracies on adversarial examples of CIFAR10 (PGD $\varepsilon\!=\!4/255$ and 100 iterations)
 for various logit regularization methods.
Note that we normalized CIFAR10 such that their pixel values are in [0,1].
We used logit squeezing (LSQ), label smoothing (LSM), and bounded logits by tanh with
various $\alpha$, $\lambda$, and $\gamma$. 
The detailed experimental conditions are provided in appendix~\ref{excond}.

Figure~\ref{LogitVSRob} shows adversarial robustness against the average norms 
of logit vectors. 
Note that results of BLF in \rfig{LogitVSRob} are discussed in the next section.
In this figure, each point corresponds to each hyper-parameter.
For tanh,
we also plot adversarial robustness against average norms of pre-logit vectors $\bm{z}$ as well as logit vectors $\gamma g(\bm{z})$. 
In \rfig{LogitVSRob}, models learned using logit regularization methods have various norms and robustness
depending on $\alpha$ and $\lambda$, 
and the robust accuracies become almost zero when norms exceed about seven.
The results of tanh indicate that even if models have small logits by adding 
bounded functions,
they are vulnerable when their pre-logit norms are large.
In addition, the results of tanh imply that just scaling logits $\gamma \bm{z}$
does not improve the robustness.
From the observation, we need a new function that has bounded outputs and inputs. 
\begin{figure}[tbp]
\centering
\includegraphics[width=.9\linewidth]{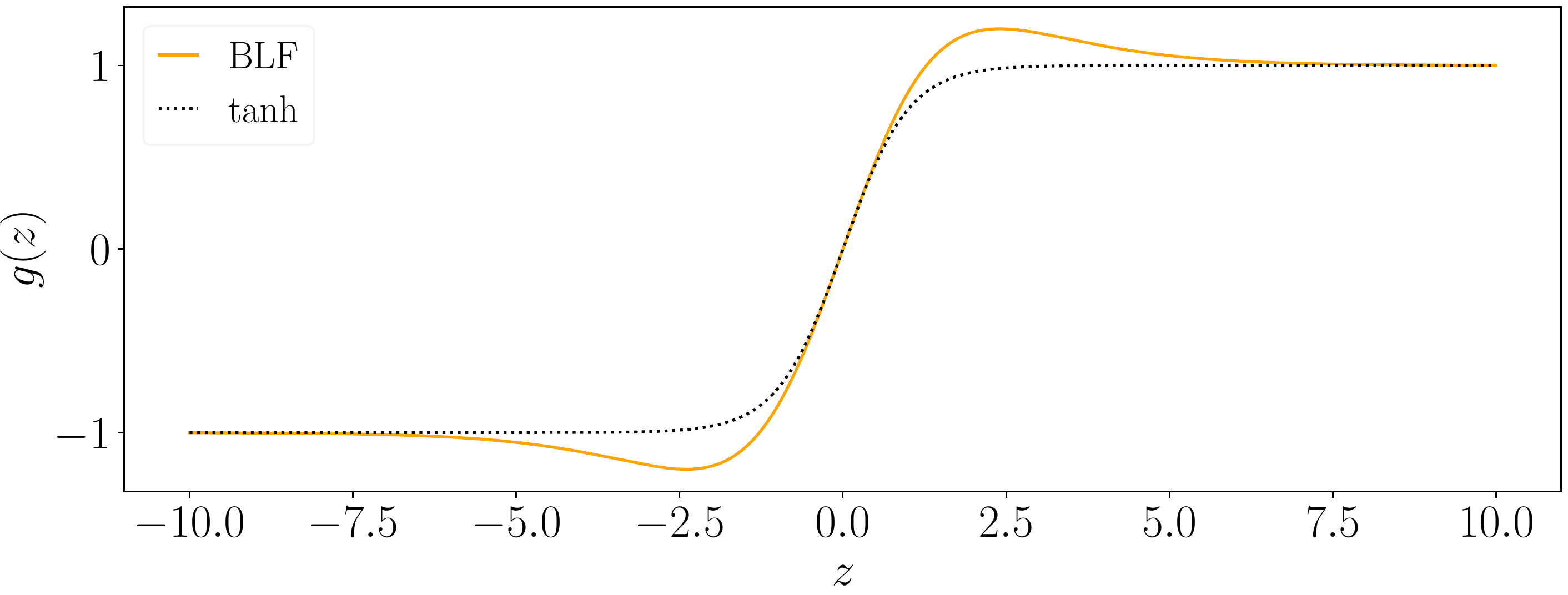}
\caption{Comparison of BLF and tanh}
\label{propFun}
\end{figure}
\section{Proposed method}
We propose constraining logits by the addition of
 a new activation function called bounded logit function (BLF) just before softmax.
BLF is defined as follows:
\begin{definition}
Bounded logit function (BLF) is defined as 
\begin{align}
g(z)=2\left\{z\sigma(z)+\sigma(z)-z\sigma^2(z)\right\}-1,
\label{propEq}
\end{align}
where $\sigma$ is sigmoid function. When $z$ is a vector, BLF becomes element-wise operation.
\end{definition}
BLF is similar to tanh as shown in \rfig{propFun}. 
In fact, BLF has 
some of the same properties of tanh, e.g.,
$\lim_{z\rightarrow+\infty}g(z)=1$, $\lim_{z\rightarrow-\infty}g(z)\!=\!-1$,
and $\left.\partial g(z)/\partial z \right|_{z=0}\!\!=\!z$.
However, this function has the maximum and minimum points in $-\sqrt{5}\!-\!1\!<\!z\!<\!-2$ 
and $2\!<\!z\!<\!\sqrt{5}\!+\!1$ while tanh does not have the finite maximum and minimum points.
Therefore, we have the following theorem:
\begin{theorem}
\label{PropPro}
Let $g(z)$ be BLF and $\gamma$ be a hyper-parameter satisfying $0\!<\!\gamma\!<\!\infty$.
If we use $g(z)$ before softmax as $\bm{f}_s(\gamma g(\bm{z}_{\bm{\theta}}(\bm{x})))$, 
 all elements of the optimal pre-logit vector
 $\bm{z}_{\bm{\theta}^*}(\bm{x}^{(i)})\!=\!\mathrm{arg}\!\min_{\bm{z}_{\bm{\theta}}} \mathcal{L}
_{\mathrm{CE}}(\gamma g(\bm{z}_{\bm{\theta}}(\bm{x}^{(i)})),\bm{p}^{(i)}) $
have finite values, and all elements of the optimal logit vector
 $\gamma g(\bm{z}_{\bm{\theta}^*})$ also have finite values.
Specifically, we have the following equalities and inequalities:
\begin{align}\textstyle
\gamma g(z_{\bm{\theta}^*,k}(\bm{x}^{(i)}))&\textstyle=
\begin{cases}
\gamma\max_z g(z)~~~k=t\\\textstyle
\gamma\min_z g(z)~~~\mathrm{otherwise},
\end{cases}\\
\textstyle\gamma<&\textstyle |\gamma g(z_{\bm{\theta}^*,k}(\bm{x}^{(i)}))|<\gamma\frac{\sqrt{5}+1}{2},\nonumber\\
z_{\theta^*,k}(\bm{x}^{(i)})&\textstyle=\begin{cases}\textstyle\mathrm{arg}\!\max_z g(z)~~~k=t\\\textstyle
\mathrm{arg}\!\min_z g(z)~~~\mathrm{otherwise},
\end{cases}\\
\textstyle 2<&\textstyle|z_{\bm{\theta}^*,k }(\bm{x}^{(i)})|<\sqrt{5}+1.\nonumber
\end{align}
\end{theorem}
Theorem~\ref{PropPro} indicates that
this function gives finite logits and pre-logits\footnote{If we need more specific values of the optimal logits and inputs, they can be solved numerically.}
 unlike tanh.
Therefore, we can keep both logits and pre-logits as small values when we add BLF before softmax.
We can control the scale of logits by using the hyper-parameter $\gamma$.
We conducted experiments using various $\gamma$ in the same manner as mentioned in the previous section,
 and evaluated adversarial robustness against norms of logits and pre-logits (\rfig{LogitVSRob}).
From this figure, BLF keeps the logits and pre-logits small, and
its robust accuracy is higher than that of logit squeezing.
Furthermore, since the optimal pre-logits $z_{\bm{\theta}^*,j}(\bm{x}^{(i)})$ do not depend on $\gamma$,
pre-logits of BLF on some $\gamma$ can have almost the same norms; the $L_\infty$ norms are vertically
aligned on about 2.4 despite the difference in $\gamma$. 
The result indicates that the empirical optimal pre-logits follow Theorem 3,
though our theoretical results are based on the Assumption in Section~\ref{AsSec}.

$\gamma$ can be used as a learnable parameter. However, the optimal $\gamma$ becomes infinitely large 
by minimization of softmax cross-entropy, and the logit norms become infinite values.
Thus, the learnable $\gamma$ does not improve
adversarial robustness.
We also evaluated a learnable version of BLF (L-BLF) in the next section.
In this setting, we used $\gamma=\mathrm{softplus}(\tilde{\gamma})$ and optimized $\tilde{\gamma}$
to keep $\gamma$ non-negative.

Note that one of the reasons why we use BLF is that BLF is similar to tanh, so it is easy to verify
the effect of the finite optimal points.
We can use other bounded functions, which are not monotonically increasing,
instead of BLF. We evaluate some such functions in appendix~\ref{EvalSec}.
\section{Experiments}\label{exp}
In this section, we conducted experiments to evaluate the proposed method in terms of (a)
robustness against gradient-based attacks, 
(b) robustness against gradient-free attacks, 
and (c) operator norms of models.
In the experiments, we evaluated the models using only clean training data (standard training)
and using adversarially perturbed training data (adversarial training), respectively.
\subsection{Experimental Conditions}
This~section~gives~an~outline~of~the~experimental~conditions and 
the details are provided in appendix~\ref{excond}.
Datasets of the experiments were MNIST \cite{mnist} and CIFAR10 \cite{cifar}.
Our method was compared with a model trained without any logit regularization methods (Baseline),
logit squeezing (LSQ), and label smoothing (LSM).
In our method, we evaluated BLF with fixed $\gamma$ (BLF) and BLF with learnable $\gamma$ (L-BLF).
We also compared them with TRADES,
which is a strong defense method using adversarial examples \cite{TRADES}, in the adversarial training setting.

For MNIST, we used a convolutional neural network (CNN) composed of two convolutional layers 
and two fully connected layers (2C2F)
and one that is composed of four convolutional layers 
and three fully connected layers (4C3F) following \cite{TRADES}.
For CIFAR10, we used ResNet-18 (RN18) \cite{resnet} and WideResNet-34-10 (WRN) \cite{WRN} also following \cite{TRADES}.

We used untargeted projected gradient descent (PGD), which is the most popular white box attack,
as a gradient-based attack and used SPSA and Square Attacks as gradient-free attacks.
The hyper-parameters for PGD were based on \cite{pgd2}.
The $L_\infty$ norm of the perturbation $\varepsilon$ was set to $\varepsilon\!=\!0.3$ for MNIST
and $\varepsilon\!=\!8/255$ for CIFAR10 at training time.
For PGD, we randomly initialized the perturbation and updated it for 40 iterations with a step size of 0.01 on MNIST
at training and evaluation times,
and on CIFAR10 for 7 iterations with a step size of 2/255
at training time and 100 iterations with the same step size at evaluation time.
At evaluation time, we use $\varepsilon\!=\![0, 0.05, \dots, 0.3]$ on MNIST
and $\varepsilon\!=\![0, 2/255, \dots, 20/255]$ on CIFAR10.
$\varepsilon\!=\!0$ corresponds to clean data.
For TRADES, we set hyper-parameters of adversarial examples based on the code provided
by the authors.\footnote{https://github.com/yaodongyu/TRADES}
On MNIST, step size was set to 0.01,
and the number of steps was set to 40, and $\varepsilon$ was set to 0.3.
On CIFAR10, step size was set to 2/255,
and the number of steps was set to 10, and $\varepsilon$ was set to 8/255.
We selected the best hyper-parameters of our method $\gamma$, logit squeezing $\lambda$, label smoothing $\alpha$, and TRADES $\beta$
among five parameters. The selected hyper-parameters are shown in Figs.~\ref{MNIST-2CNN} and \ref{CIFAR10Fig}
for MNIST and CIFAR10, respectively.
For WRN, we used the same hyper-parameters as those of RN18. 
We trained models for five times for MNIST and three times for CIFAR10 and
show the average and standard deviation of test accuracies.
To generate adversarial examples, we used advertorch \cite{ding2018advertorch}.
\subsection{Robustness against gradient-based methods}
\subsubsection{Accuracy against PGD}
Figures~\ref{MNIST-2CNN}~(a)-(d) show accuracies on MNIST attacked by PGD.
In these figures, results of $\varepsilon\!=\!0$ correspond to clean accuracy.
For 2C2F (\rfig{MNIST-2CNN}(a)) in the standard training setting, label smoothing is the most 
robust against PGD until $\varepsilon$ is smaller than 0.2.
For 2C2F and 4C3F (Figs.~\ref{MNIST-2CNN}(a) and (c)) with $\varepsilon>0.15$, BLF is the most robust in standard training.
When we use the learnable $\tilde{\gamma}$, L-BLF does not improve the robustness.
This is because $\tilde{\gamma}$ becomes large to minimize the loss function,
and norms of logits have large values.
In the adversarial training setting, our proposed function improves
robustness the most for 2C2F and is comparable to TRADES for 4C3F.
Note that TRADES of 2C2F (\rfig{MNIST-2CNN}(b)) is not more robust than Baseline in our experiments.
This might be because we train models on training data attacked by PGD with $\varepsilon=0.3$
following \cite{pgd2}
while \citet{TRADES} train them on training data attacked by PGD with $\varepsilon=0.1$ in Table~4 of \cite{TRADES}.

Figure~\ref{CIFAR10Fig} shows the results on CIFAR10.
In the standard training setting, label smoothing improves robustness the most for RN18, and
our proposed method improves robustness the most for WRN.
In the adversarial training setting, our method improves the robustness the most.
It is more robust than TRADES 
even though the number of iteration for TRADES is larger than that of our method.

The reason label smoothing and logit squeezing are not so effective
in the adversarial training might be due to the complexity of the objective function:
regularization terms might disturb the mini-max problem for adversarial robustness.
On the other hand, our method does not change the objective function.
Thus, it is more suitable for adversarial training.
\begin{figure}[tb]
\centering
\subfloat[Standard training]{\includegraphics[width=.5\linewidth]{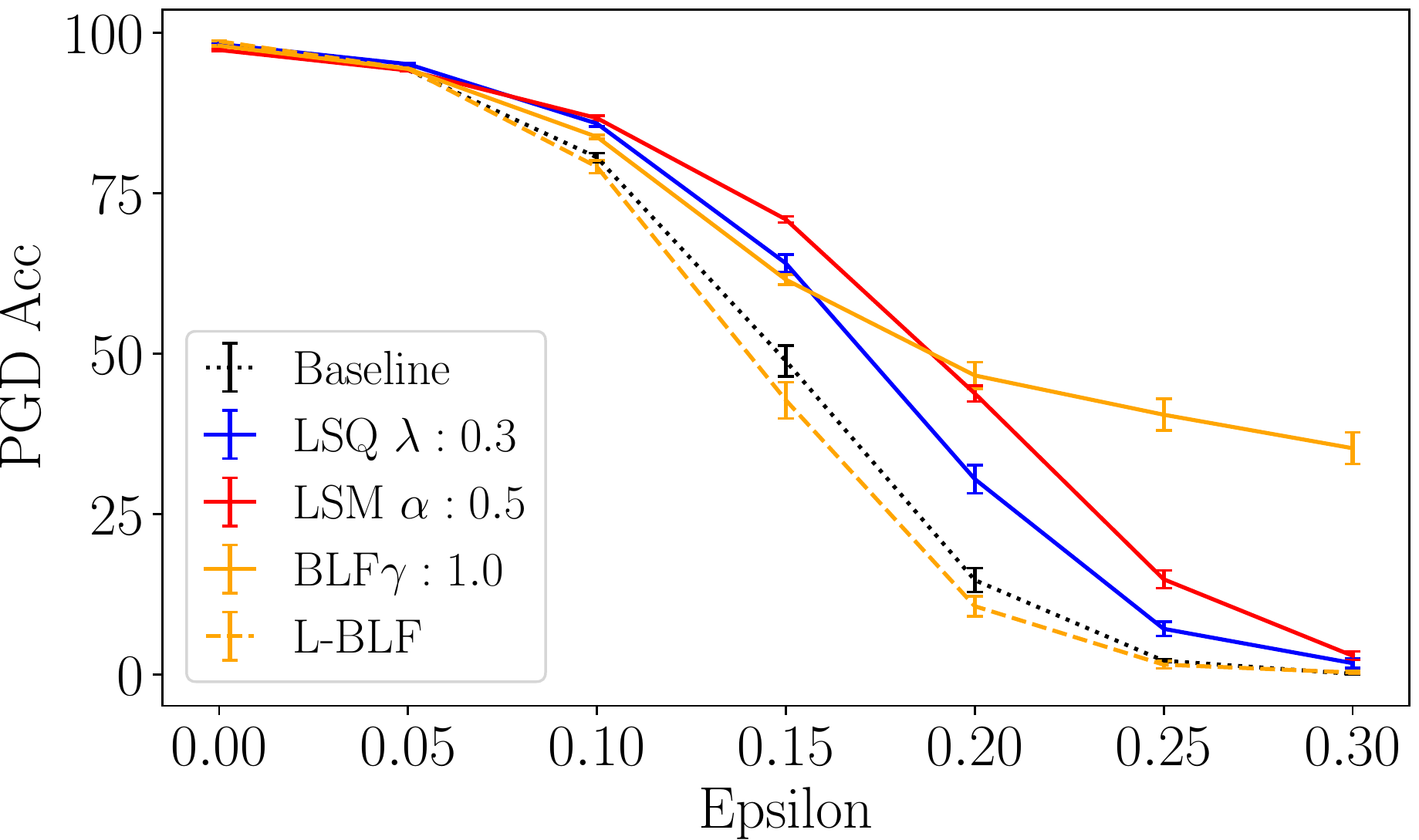}}
\centering\hfill
\subfloat[Adversarial training]{\includegraphics[width=.5\linewidth]{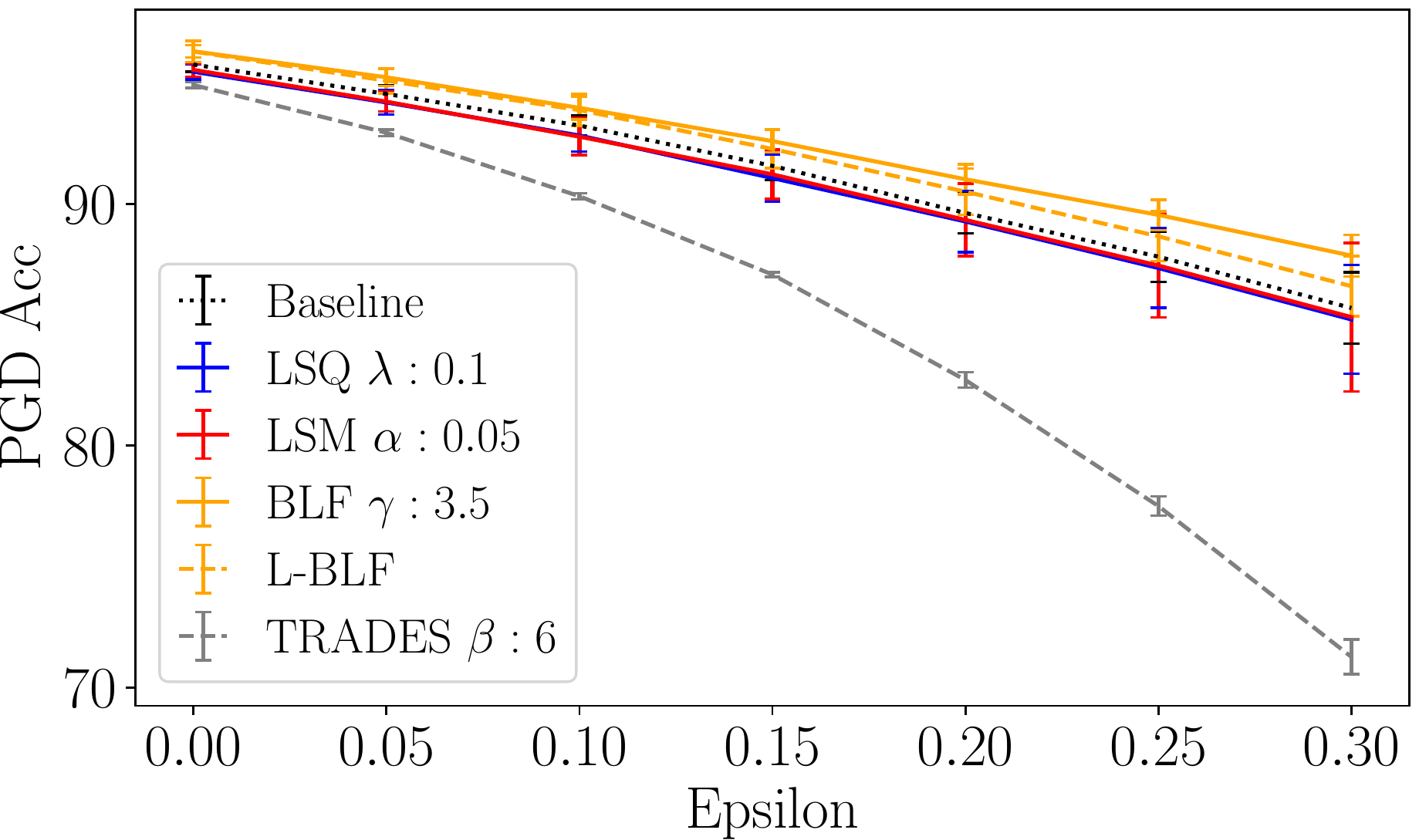}}\\
\centering
\subfloat[Standard training]{\includegraphics[width=.5\linewidth]{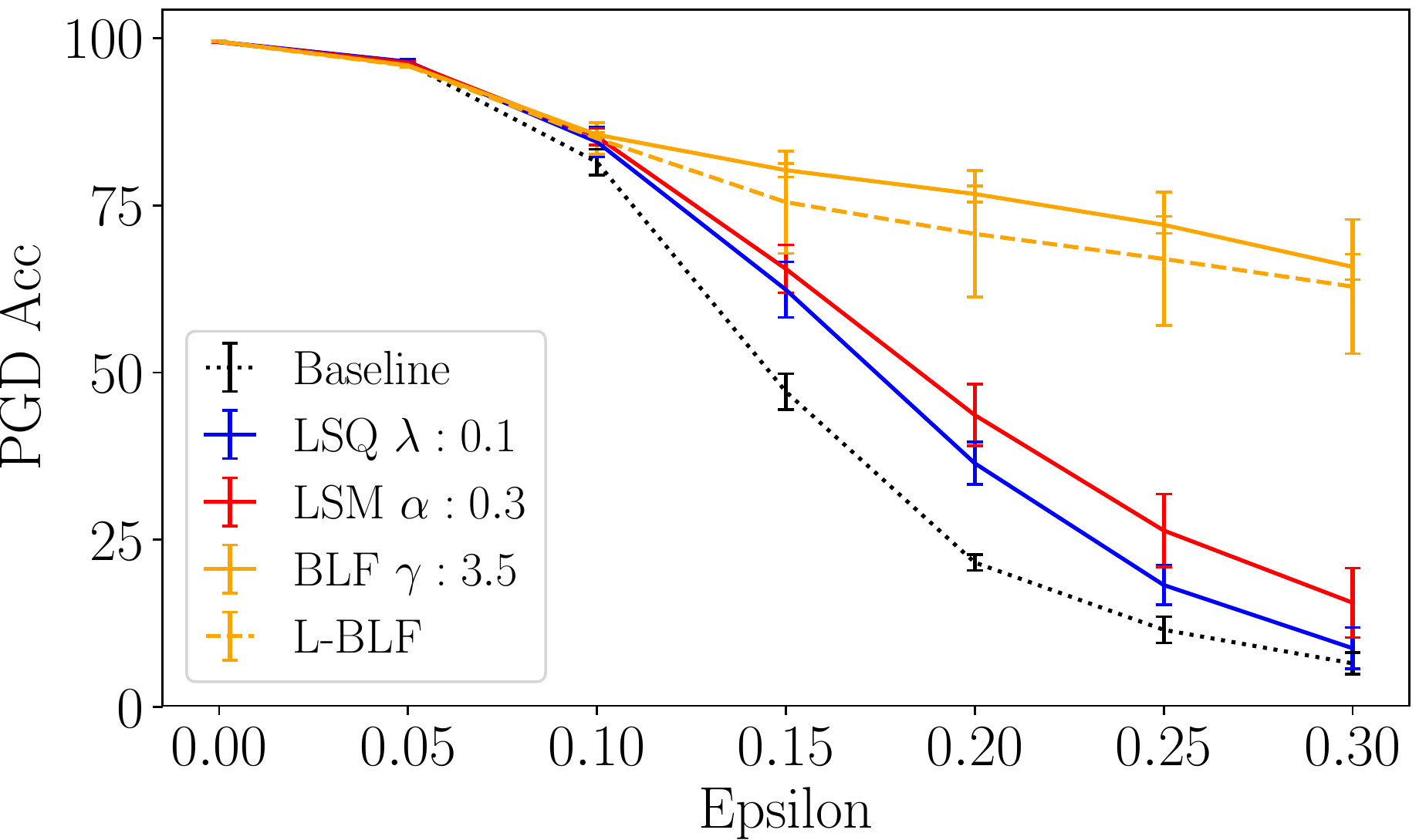}}
\centering\hfill
\subfloat[Adversarial training]{\includegraphics[width=.5\linewidth]{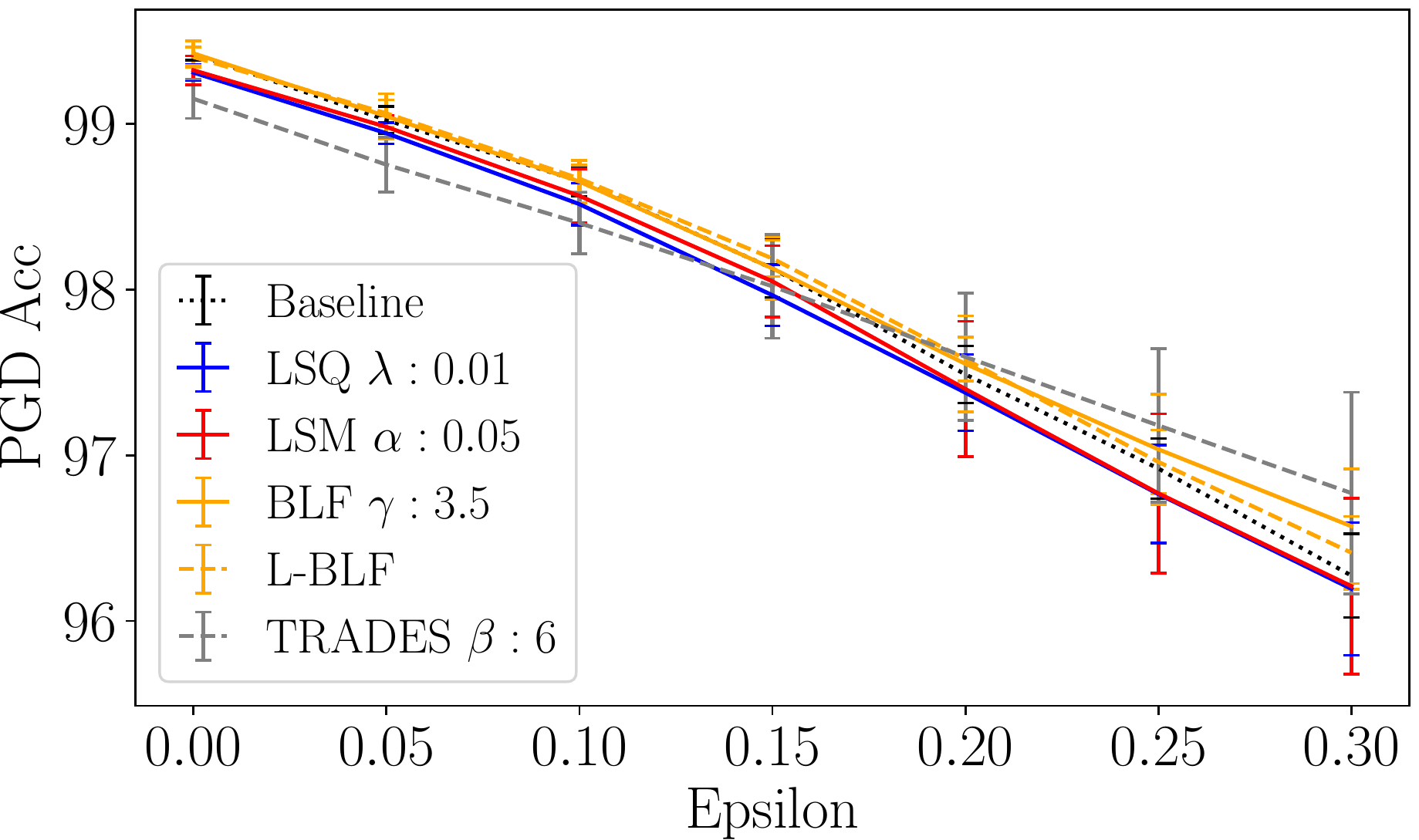}}
\caption{Accuracy of 2C2F (top) and 4C3F (bottom) on MNIST attacked by PGD (40 iterations). 
Error bars correspond to standard deviations.}
\label{MNIST-2CNN}
\end{figure}
\begin{figure}[tb]
\centering
\subfloat[Standard training]{\includegraphics[width=.5\linewidth]{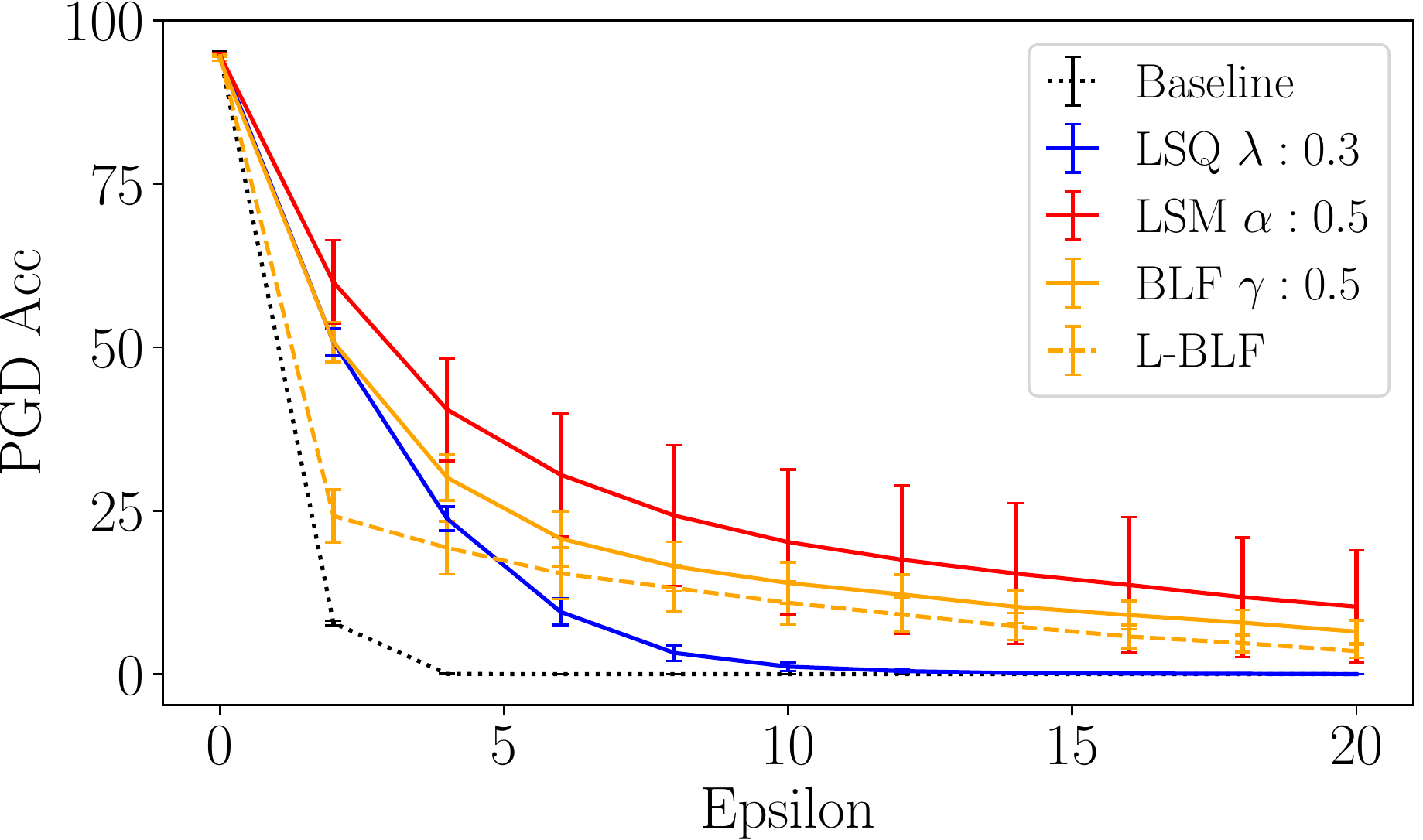}}
\centering\hfill
\subfloat[Adversarial training]{\includegraphics[width=.5\linewidth]{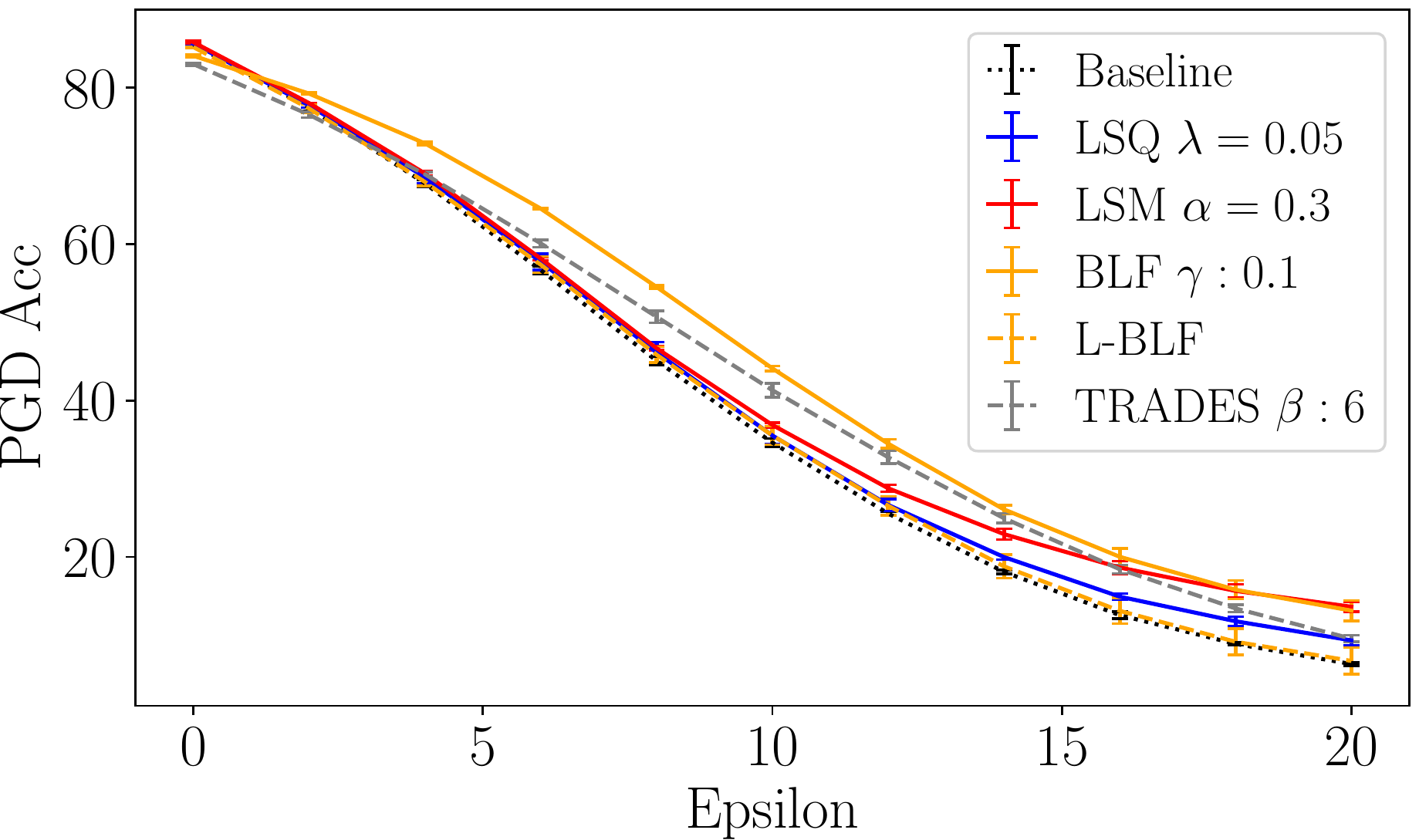}}\\
\centering
\subfloat[Standard training]{\includegraphics[width=.5\linewidth]{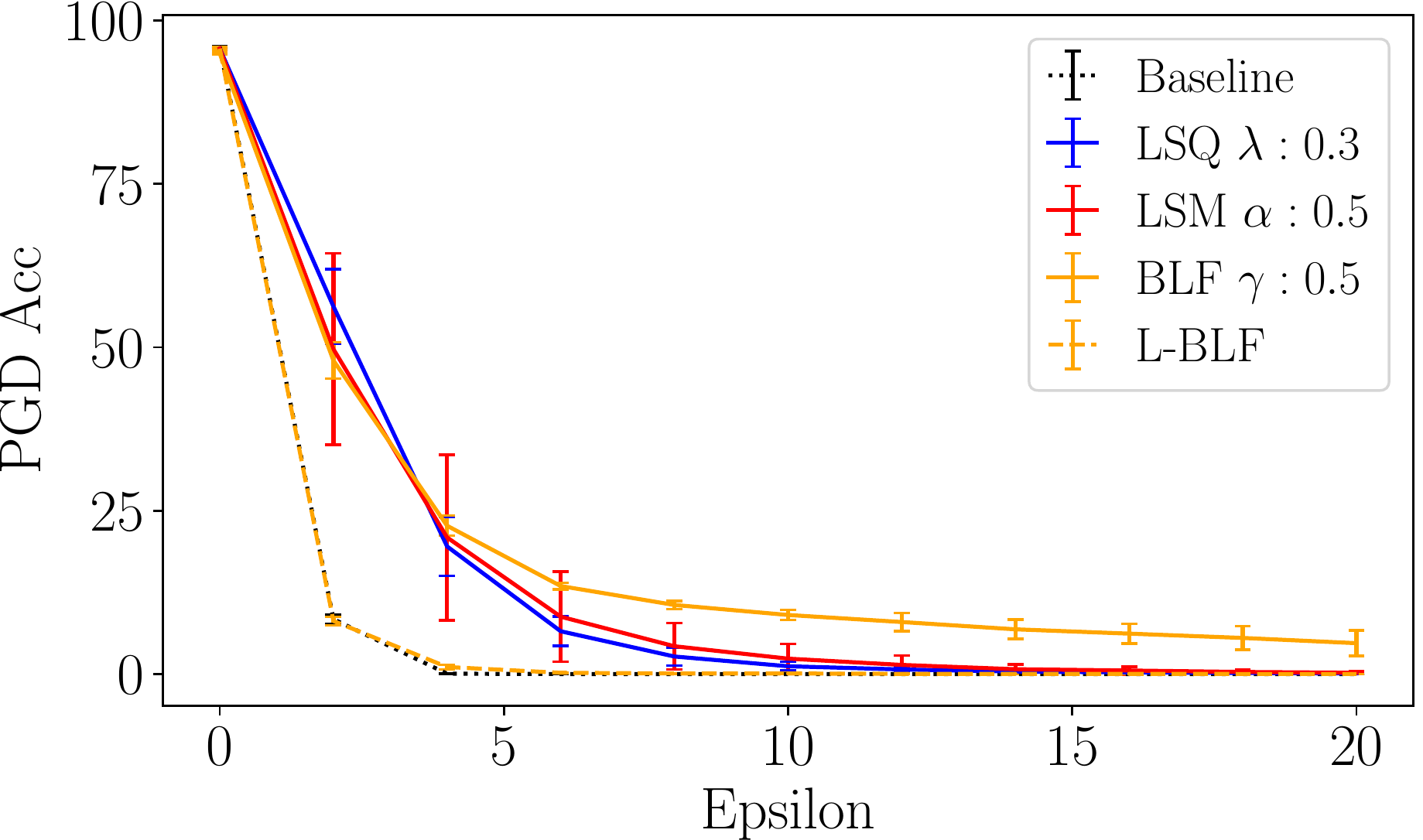}}
\centering\hfill
\subfloat[Adversarial training]{\includegraphics[width=.5\linewidth]{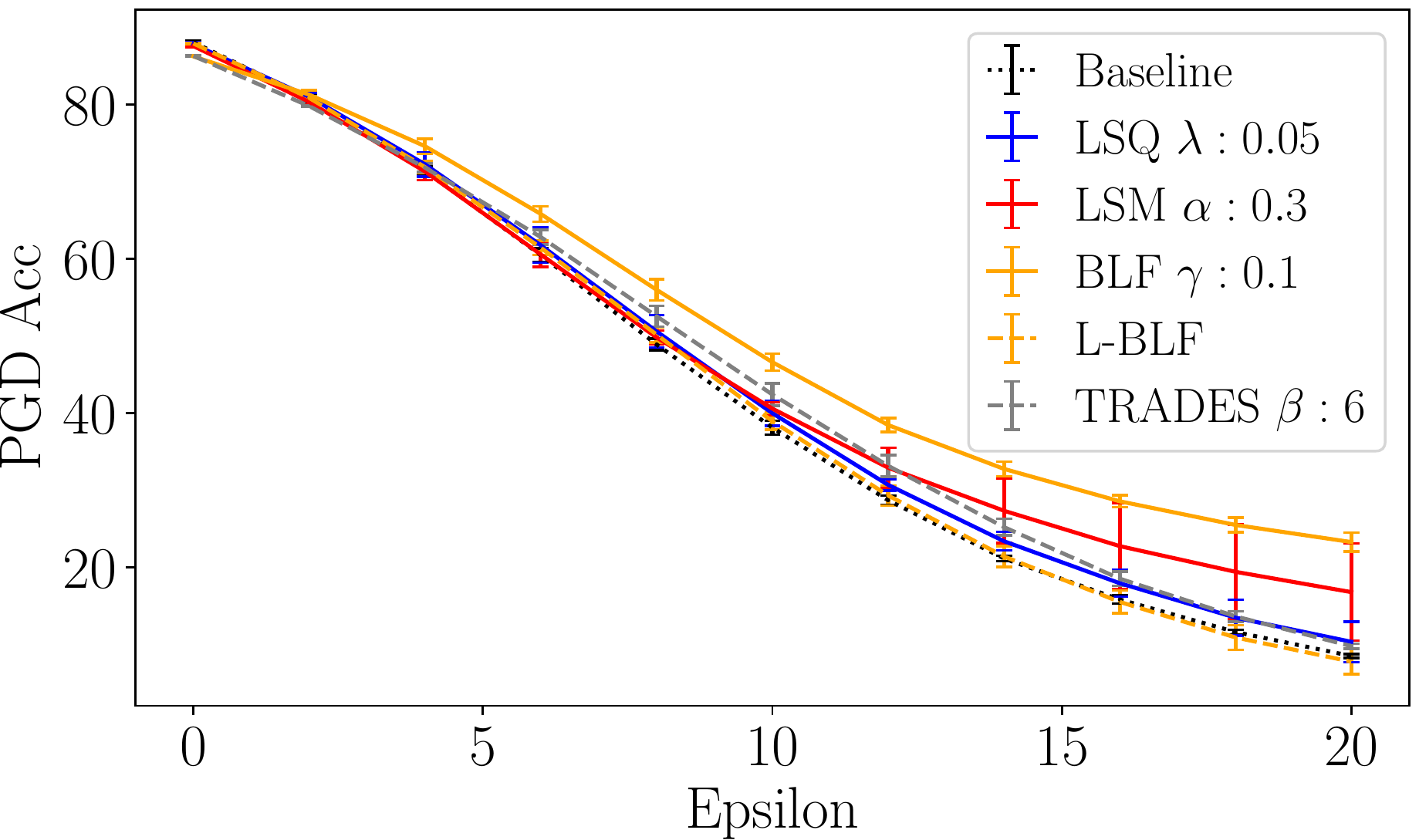}}
\caption{Accuracy of RN18 (top) and WRN (bottom) on CIFAR10 attacked by PGD (100 iterations). 
Error bars correspond to standard deviations.}
\label{CIFAR10Fig}
\end{figure}

\subsubsection{Evaluation of misleading gradients  of BLF}
    As shown in \rfig{propFun}, the absolute values of BLF $|g(z)|$ does not become
    smaller than one in intervals of $z\!<\!\mathrm{arg}\!\min_z g(z)$ and $z\!>\!\mathrm{arg}\!\max_z g(z)$.
    In the intervals, the gradient of BLF might mislead the gradient-based attacks
    because the absolute values of outputs of BLF only change in $1\!<\!|g(z)|\!<\!\max_z |g(z)|$.
    This might be a cause of robustness of BLF, which is not expected.
    To investigate the effect of the intervals, 
    we replaced a BLF with tanh to generate PGD attacks.
    We conducted this experiment by the following procedure: (i) we trained BLF models
    in standard and adversarial settings, (ii) we replaced BLF with tanh in models trained at the previous step, 
    (iii) we generated PGD attacks by using replaced models, (iv) we replaced tanh with BLF again and evaluated robust accuracies of BLF models against PGD attacks generated at the previous step.
    Since tanh is similar to BLF and is a monotonically increasing function,
    gradient-based attacks can effectively avoid the misleading gradients in the above mentioned intervals.
    In fact, we observed that replaced models achieved almost the same clean accuracy as BLF models even though
    their parameters are optimized for BLF: Clean accuracies of BLF (standard training), replaced models (standard training)), BLF (adversarial training),
    and replaced models (adversarial training) are 94.58, 94.58, 82.42, and 83.38, respectively.
    
    Robust accuracies of RN18 on CIFAR10 is shown in \rfig{RefPGD}.
    In this figure, we show robust accuracies of BLF models against PGD by using replaced models
    and using the BLF models.
    Though PGD for replaced models succeeded in attacking BLF models more than PGD for BLF models when $\varepsilon$ is set to be greater than $10$ in standard training settings,
    it could not attack BLF models well in the adversarial training setting and in the standard training setting when $\varepsilon <10$.
    Therefore, BLF does not only employ the misleading gradients in $z\!<\!\mathrm{arg}\!\min_z g(z)$ and $z\!>\!\mathrm{arg}\!\max_z g(z)$ for improving adversarial robustness.
    \begin{figure}[tb]
        \centering
        \subfloat[Standard training]{\includegraphics[width=.5\linewidth]{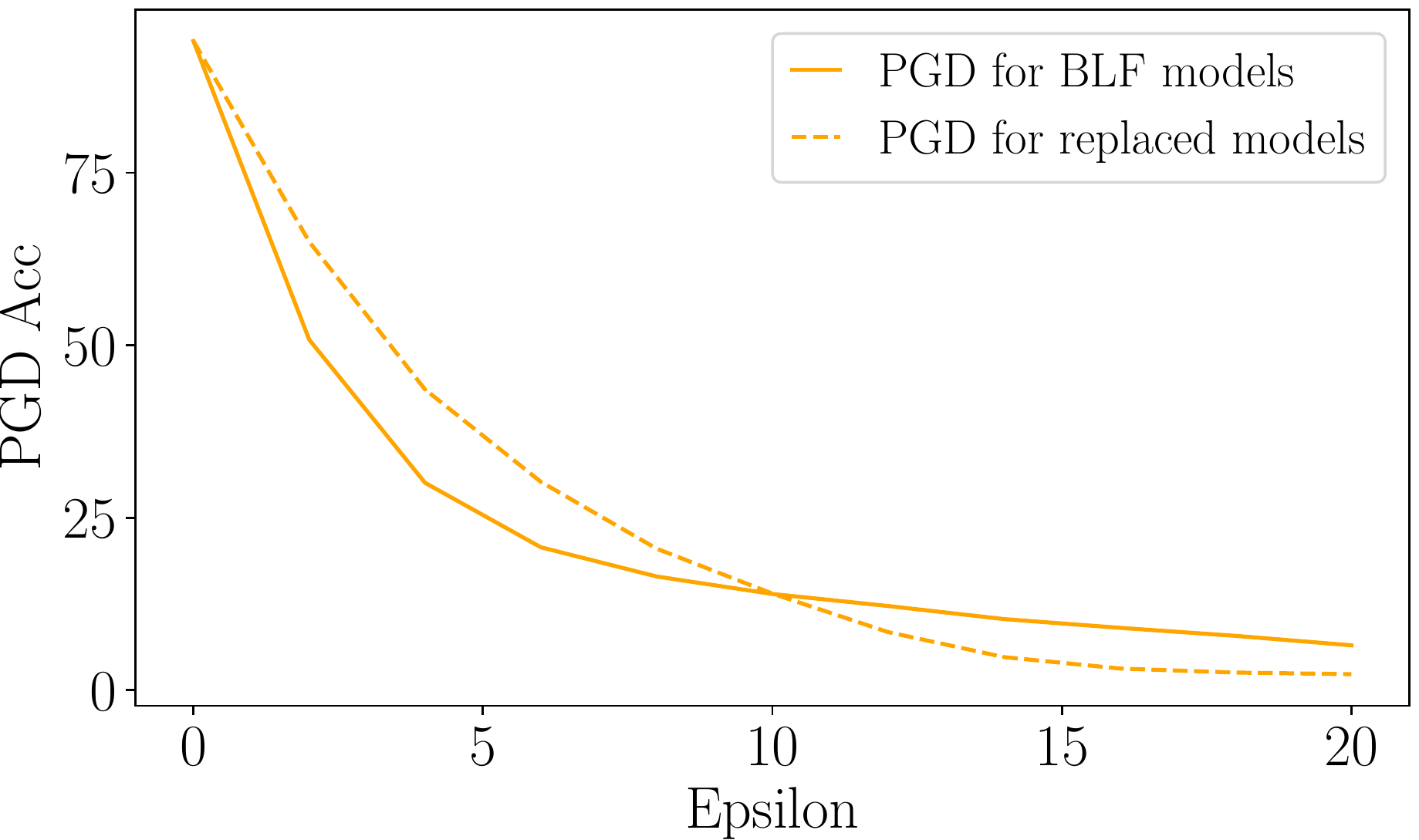}}
        \centering
        \subfloat[Adversarial training]{\includegraphics[width=.5\linewidth]{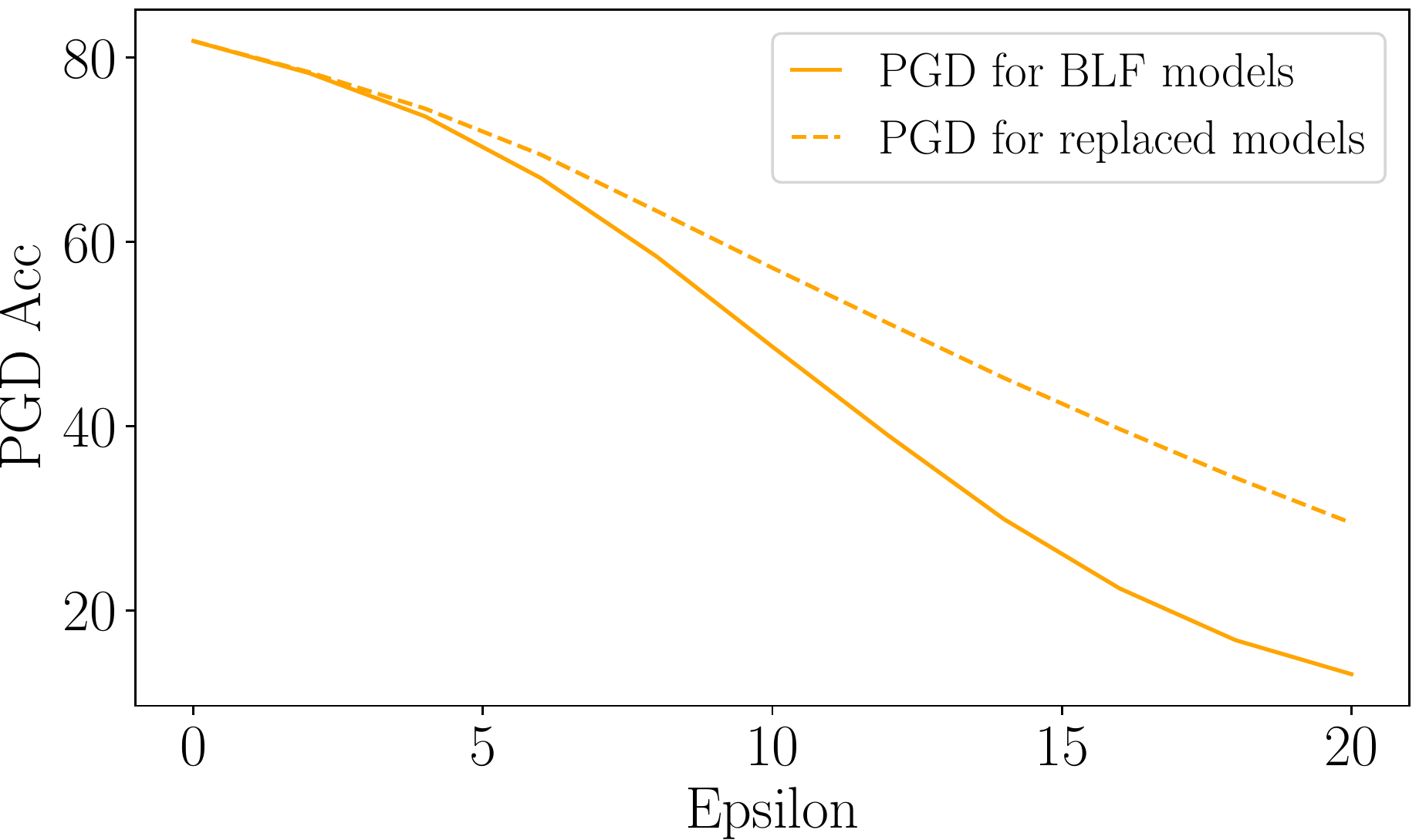}} 
        \caption{Robust accuracy against PGD by using models replaced tanh. We compare robust accuracies
        against PGD naively generated by using BLF models and against PGD generated by using tanh.}
        \label{RefPGD}
    \end{figure}

\subsection{Robustness against gradient-free attacks}
Since \citet{mosbach2018logit} pointed out that logit regularization
only masks or obfuscates gradient to improve robustness,
we evaluated the robustness against gradient-free attacks.
In the experiments, we tuned hyperparameters ($\lambda$, $\alpha$, $\gamma$, $\beta$)
for each attack and train the model for one time for each hyper-parameter.
\subsubsection{Accuracy against SPSA attacks}
We used SPSA as a gradient-free attack because it can operate
when the loss surface is difficult to optimize \cite{carlini2019evaluating,SPSA}.
We set the hyper-parameters of SPSA as epsilons of 0.15 for MNIST and 8/255 for CIFAR10, perturbation size of 0.01, Adam learning rate 
of 0.01, maximum iterations of 40 for MNIST and 10 for CIFAR10, and batchsize of 2048.\footnote{We could not use the original hyper-parameters \cite{SPSA}
since SPSA requires high computation costs \cite{shafahibatch}. Even so, we could evaluate the robustness of our method relatively.}
 
Table~\ref{SPSAtab} lists the accuracies of
2C2F and 4C3F with each method on MNIST attacked by SPSA and
RN18 and WRN with each method on CIFAR10 attacked by SPSA.
We can see that in the standard training setting, 
BLF improves robustness against SPSA the most on MNIST even though
robustness against PGD of BLF is lower than those of label smoothing and logit squeezing in \rfig{MNIST-2CNN} (a).
On CIFAR10, label smoothing and logit squeezing improve robustness more than BLF.
In adversarial training settings, however, our method improves robustness the most in a majority of settings.
These results are in agreement with the results of gradient-based attacks.
\begin{table}[tbp]
    \centering
    \caption{Robust accuracies against SPSA on MNIST ($\varepsilon\!=\!0.15$)  and CIFAR10 ($\varepsilon\!=\!8/255$).
    ST and AT represent standard training and adversarial training, respectively.}
    \label{SPSAtab}
        \small
    \begin{tabular}{ccccccc}\toprule
    &Baseline&LSQ&LSM &BLF&TRADES\\\midrule
    2C2F ST&$53.5$&$60.8$&$67.4$&$\bm{68.4}$&N/A\\
    2C2F AT&$90.2$&$92.3$&$90.0$&$\bm{92.7}$&$86.4$\\\midrule
    4C3F ST&$51.3$&$69.2$&$73.8$&$\bm{82.4}$&N/A\\
    4C3F AT&$\bm{98.3}$&$\bm{98.3}$&$98.2$&$98.2$&$98.1$\\\midrule\midrule
    RN18 ST&$0.3$&$23.5$&$\bm{24.9}$&$23.4$&N/A\\
    RN18 AT&$56.6$&$56.8$&$56.6$&$\bm{58.7}$&$57.4$\\\midrule
    WRN ST&$0.1$&$\bm{42.1}$&$16.3$&$13.9$&N/A\\
    WRN AT&$60.6$&$61.5$&$60.0$&$\bm{62.3}$&$61.3$\\
    \bottomrule
    \end{tabular}
\end{table}

\begin{table}[tbp]
    \centering
    \caption{Robust accuracies against Square Attack on MNIST ($\varepsilon\!=\!0.15$)  and CIFAR10 ($\varepsilon\!=\!8/255$).
    ST and AT represent standard training and adversarial training, respectively.}
    \label{SQUAREtab}
    \small
    \begin{tabular}{ccccccc}\toprule
    &Baseline&LSQ &LSM&BLF&TRADES\\\midrule
    2C2F ST& 51.5 &     53.0 &      $\bm{64.0}$ &  59.4&N/A\\
    2C2F AT&91.4 &     90.3 &      91.0 &  $\bm{92.1}$ &  85.3\\\midrule
    4C3F ST&35.6 &     46.2 &      51.5 &  $\bm{56.1}$ &N/A\\
    4C3F AT&97.9 &     $\bm{98.0}$ &      97.9 &  $\bm{98.0}$ &  $\bm{98.0}$\\\midrule\midrule
    RN18 ST&$0.3$&$20.1$&$27.6$&$\bm{31.7}$&N/A\\
    RN18 AT&$54.5$&$55.4$&$\bm{55.6}$&$53.8$&$55.3$\\\midrule
    WRN ST&$0.2$&30.2 &      $\bm{38.4}$ &   30.8 &N/A\\
    WRN AT&$59.7$&$59.7$&$60.0$&$\bm{60.9}$&$59.4$\\
    \bottomrule
    \end{tabular}
\end{table}
\subsubsection{Accuracy against Square Attacks}
Although the SPSA attack does not use exact gradients, it still approximates gradients to generate attacks.
Thus, obfuscating gradients might be still effective for the SPSA attack.
To investigate whether logit regularization methods just obfuscate gradients,
we evaluate the robustness against Square Attack \cite{SQUARE}, which is a query-based
black box attack.
Since the Square Attack uses random search to generate attacks,
obfuscating gradients are ineffective for the Square Attack.
To generate Square Attacks, we set the number of queries to 5000 and use the code in \cite{AutoAttack}.
Note that untargeted Square Attacks use a margin loss instead of cross entropy loss.

Robust accuracies against Square Attacks are listed in \rtab{SQUAREtab}.
In this table, BLF achieves the highest or the second highest accuracies on almost all settings.
In addition, all logit regularization methods without adversarial training can improve
the robust accuracies though Square Attacks do not use gradients.
Therefore, logit constraints does not only just obfuscate gradients for improving robustness.
\subsection{Evaluation of operator norms}
As discussed in Section~\ref{logitSec}, softmax cross-entropy can cause large Lipschitz constants,
and it might be a cause of vulnerabilities.
To investigate Lipschitz constants of models, we computed averages of $L_\infty$ 
operator norms of convolution layers of RN18 (\rtab{Weighttab}) by following 
\cite{gouk2018regularisation}.
The $L_\infty$ operator norms of convolution layers can be a criterion of Lipschitz constants
since one of Lipschitz constants of composite functions is the product of Lipschitz constants of composing functions
and $L_\infty$ operator norm is a Lipschitz constant for a linear function.
Table~\ref{Weighttab} shows that logit regularization methods induce small $L_\infty$ operator norms 
of convolution layers compared with Baseline
 even though they do not explicitly impose the penalty of parameter values.
This table indicates that
BLF can outperform other methods when using adversarial training
because it effectively induces small Lipschitz constants.
On the other hand, the $L_\infty$ norm of L-BLF is almost the same as
that of Baseline. Thus, BLF with learnable $\gamma$ does not improve the robustness. 
Note that the $L_\infty$ norm of Baseline does not become extremely large 
because we applied some regularization methods, e.g., weight decay and early stopping,
 into all methods to obtain good generalization performance.
 \begin{table}[tbp]
    \centering
    \caption{Averages of $L_\infty$ operator norms of convolution layers of RN18.}
    \label{Weighttab}
    \small
    \begin{tabular}{ccccccc}\toprule
    &Baseline&LSQ&LSM &BLF &L-BLF&TRADES\\\midrule
    ST&$19.6$&$17.1$&$10.1$&$11.2$&$20.1$&N/A\\
    AT&$17.1$&$16.5$&$11.4$&$4.4$&$16.8$&$13.3$\\
    \bottomrule
    \end{tabular}
    \end{table}
    
\section{Conclusion}
We proposed a method of constraining the logits by adding a bounded activation function just before softmax
following the hypothesis that small logits improve the adversarial robustness.
We developed a new bounded function that has the finite maximum and minimum points 
so that logits and pre-logits have small values.
Compared with other logit regularization methods,
our method can effectively improve the robustness in adversarial training
despite its simplicity.

Though we provided insights into the vulnerabilities of softmax cross-entropy
and empirical evidence of the effectiveness of logit regularization methods,
it is still an open question why small logits can improve robustness. 
Our experiments of tanh indicate that small logits are not sufficient for adversarial robustness.
Even so, our experiments showed that our method is comparable to the recent defense method
in terms of adversarial robustness 
against both gradient-based and gradient-free attacks in adversarial training.
Thus, our results indicate that the investigation into the relation between logit regularization
 and robustness is still an important research direction to reveal the cause of vulnerabilities of DNNs.

\bibliography{CameraBib.bib}

\appendix
\section{Proofs}
\label{ProfSsec}
In this section, we show proofs of theoretical results in the paper following
the assumption.
\begin{assumption*}
    We assume that (a) if data points have the same values as $\bm{x}^{(i)}\!=\!\bm{x}^{(j)}$, they have 
    the same labels as $\bm{p}^{(i)}\!=\!\bm{p}^{(j)}$,
    (b) the logit vector $\bm{z}_{\bm{\theta}}(\bm{x})$ can be an arbitrary vector for each data point, and
    (c) the optimal point $\bm{\theta}^*\!=\!\mathrm{arg}\!\min_{\bm{\theta}}\!\frac{1}{N}\!\sum_{i=1}^{N}\! \mathcal{L}(\bm{z}_{\bm{\theta}}(\bm{x}^{(i)}),\bm{p}^{(i)})$ achieves $\mathcal{L}(\bm{z}_{\bm{\theta}^*}(\bm{x}^{(i)}),\bm{p}^{(i)})\!=\min_{\bm{\theta}}\mathcal{L}(\bm{z}_{\bm{\theta}}(\bm{x}^{(i)}),\bm{p}^{(i)})$ for all $i$.
    \end{assumption*}
    \setcounter{theorem}{0}
    \begin{theorem}
        % \label{SCETh}
        If we use softmax cross-entropy and one-hot vectors as target values, 
        at least one element of the optimal logits $\bm{z}_{\bm{\theta}^*}(\bm{x}^{(i)})\!=\!\mathrm{arg}\min_{\bm{z}_{\bm{\theta}}} \mathcal{L}
        _{\mathrm{CE}}(\bm{z}_{\bm{\theta}}(\bm{x}^{(i)}),\bm{p}^{(i)}) $
        does not have a finite value.
        \end{theorem}
\begin{proof}
Let $t$ be a target label for $\bm{x}^{(i)}$. 
The objective function of softmax cross-entropy loss for $\bm{x}^{(i)}$ is
\begin{align}
\textstyle
J=
\mathcal{L}_{\mathrm{CE}}(\bm{z}_{\bm{\theta}}(\bm{x}^{(i)}),\bm{p}^{(i)})
=-\mathrm{log} [\bm{f}_s(\bm{z}_{\bm{\theta}}(\bm{x}^{(i)}))]_t.\nonumber
\end{align}
At the minimum point, we have $\left.\frac{\partial J}{\partial \bm{z}_{\bm{\theta}}(\bm{x}^{(i)})}\right|_{\bm{z}_{\bm{\theta}^*}}=\bm{0}$ since we assume $\bm{z}_{\bm{\theta}}(\bm{x}^{(i)})$
can be an arbitrary vector in Assumption. 
By differentiating softmax cross-entropy, we obtain 
\begin{align}
\textstyle
\frac{\partial J}{\partial z_{\bm{\theta},k}(\bm{x}^{(i)})}=
\begin{cases}
-1+[\bm{f}_s(\bm{z}_{\bm{\theta}}(\bm{x}^{(i)}))]_k&k\textstyle=t,\\\textstyle
[\bm{f}_s(\bm{z}_{\bm{\theta}}(\bm{x}^{(i)}))]_k&\mathrm{otherwise}.
\end{cases}
\label{sceopteq}
\end{align}
Thus, $\left.\frac{\partial J}{\partial \bm{z}_{\bm{\theta}}(\bm{x}^{(i)})}\right|_{\bm{z}_{\bm{\theta}^*}}\!\!=\!\bm{0}$ means that a softmax output vector $\bm{f}_s(\bm{z}_{\bm{\theta}^{*}}(\bm{x}^{(i)}))$ is a
one hot-vector; $ [\bm{f}_s(\bm{z}_{\bm{\theta}^{*}}(\bm{x}^{(i)}))]_t\!=\!1$ and 
$ [\bm{f}_s(\bm{z}_{\bm{\theta}^{*}}(\bm{x}^{(i)}))]_k\!=\!0$ for $k\!\neq\! t$.
Since $[\bm{f}_s(\bm{z}_{\bm{\theta}}(\bm{x}^{(i)}))]_t \neq 0$, 
we have $0<\sum_{m=1}^{M} \mathrm{exp}(z_{\bm{\theta},m}(\bm{x}^{(i)}))$.
The $t$-th output of softmax becomes
\begin{align}\textstyle
\frac{\mathrm{exp}(z_{\bm{\theta},t}(\bm{x}^{(i)}))}
{\sum_{m=1}^{M} \mathrm{exp}(z_{\bm{\theta},m}(\bm{x}^{(i)}))}&\textstyle=1,\nonumber\\\textstyle
\mathrm{exp}(z_{\bm{\theta},t}(\bm{x}^{(i)}))&\textstyle=\sum_{m=1}^{M} \mathrm{exp}(z_{\bm{\theta},m}(\bm{x}^{(i)})),\nonumber\\\textstyle
\sum_{m\neq t} \mathrm{exp}(z_{\bm{\theta},m}(\bm{x}^{(i)}))&\textstyle=0.\nonumber
\end{align}
Since $\mathrm{exp}(z)\geq 0$,
we have $\mathrm{exp}(z_{\bm{\theta},m}(\bm{x}^{(i)}))=0$ for $m\neq t$.
Since $\lim_{z\rightarrow -\infty} \mathrm{exp}(z)=0$,
the element of the optimal logits $z_{\bm{\theta}^*,m}(\bm{x}^{(i)})$ does not have a finite value for $m\neq t$. Therefore, at least one element of the optimal logits has an infinite value.
\end{proof}
\setcounter{corollary}{0}
\begin{corollary}
    % \label{SCECoro}
    If all elements of inputs $x_k^{(i)}$ are normalized as $0\leq x_k^{(i)}\leq 1$ and training dataset has at least two different labels,
    the optimal logit function $\bm{z}_{\bm{\theta}^*}(\bm{x})$
    for softmax cross-entropy is not globally Lipschitz continuous function, i.e.,
     there is not a finite constant $C\geq 0$ as
     \begin{align}\textstyle
     ||\bm{z}_{\bm{\theta}^*}(\bm{x}^{(i)})-\bm{z}_{\bm{\theta}^*}(\bm{x}^{(j)})||_{\infty}
     &\leq C ||\bm{x}^{(i)}-\bm{x}^{(j)}||_{\infty},\nonumber\\
     \forall\bm{x}^{(i)},\bm{x}^{(j)}&\in \mathcal{X}.\nonumber
     \end{align}
    \end{corollary}   
\begin{proof}
If Corollary~\ref{SCECoro} does not hold, there is a finite constant $C$ satisfying
\begin{align}\textstyle
||\bm{z}_{\bm{\theta}^*}(\bm{x}^{(i)})-\bm{z}_{\bm{\theta}^*}(\bm{x}^{(j)})||_{\infty}
 &\leq C ||\bm{x}^{(i)}-\bm{x}^{(j)}||_{\infty}\label{sceLip}
\\ \forall\bm{x}^{(i)},\bm{x}^{(j)}&\in \mathcal{X},\nonumber
\end{align}
and $0\leq C<\infty$.
We assume that $t$ and $t'$ are labels of $\bm{x}^{(i)}$ and $\bm{x}^{(j)}$, respectively, and
$t\neq t'$.
As shown in the proof of Theorem~\ref{SCETh}, 
$z_{\bm{\theta},m}(\bm{x}^{(i)})$ for $m\neq t$ does not have finite values
and $-\infty<z_{\bm{\theta},t}(\bm{x}^{(i)})\leq\infty$.
On the other hand $z_{\bm{\theta},m}(\bm{x}^{(j)})$ for $m\neq t'$ does not have finite values.
Thus, $[\bm{z}_{\bm{\theta}^*}(\bm{x}^{(i)})-\bm{z}_{\bm{\theta}^*}(\bm{x}
^{(j)})]_t=z_{\bm{\theta},t}(\bm{x}^{(i)})-z_{\bm{\theta},t}(\bm{x}^{(j)})$ does not have finite 
values.
Thus, the left-hand side of \req{sceLip} is not finite values.
On the other hand, we have $||\bm{x}^{(i)}-\bm{x}^{(j)}||_{\infty}\leq 1$ because we assume $0\leq x_k\leq 1$.
As a result, $\infty\leq C ||\bm{x}^{(i)}-\bm{x}^{(j)}||_{\infty}\leq C$, and it contradicts 
the statement $0\leq C<\infty$, which completes the proof. 
\end{proof}
\setcounter{proposition}{0}
\begin{proposition}
    The optimal logits for label smoothing $\bm{z}_{\bm{\theta}^*}(\bm{x}^{(i)})\!=\!\mathrm{arg}\min_{\bm{z}_{\bm{\theta}}}\mathcal{L}
    _{\mathrm{CE}}(\bm{z}_{\bm{\theta}}(\bm{x}^{(i)}),\bm{p}^{(i)})$
    satisfy
    \begin{align*}
    \textstyle
    z_{\bm{\theta}^*,k}(\bm{x}^{(i)})\!=\!
    \begin{cases}
        \!\log(\frac{1\!-\!\alpha}{\alpha}\!\sum_{m\neq t}\!\mathrm{exp}((z_{\bm{\theta}^*,m}(\bm{x}^{(i)}))))&\!\!\!\!\!k=t
        \\
        \!\log(\frac{\alpha}{M\!-\!1\!-\!\alpha}\!\sum_{m\neq k}\mathrm{exp}(z_{\bm{\theta}^*,m}(\bm{x}^{(i)})))&\!\!\!\!\!k\neq t.
    \end{cases}
    \end{align*}
    If an element of $\mathrm{exp}(\bm{z_{\bm{\theta}^*}}(\bm{x}^{(i)}))$ has a finite
    value, all elements of $\bm{z}_{\bm{\theta}^*}(\bm{x}^{(i)})$ have finite
    values.
\end{proposition}
\begin{proof}
Let $t$ be a label of  data point $\bm{x}^{(i)}$.
The objective function of label smoothing for $\bm{x}^{(i)}$ is 
\begin{align}
    \textstyle J&\textstyle =-\sum_m p_m \mathrm{log} [\bm{f}_s(\bm{z}_{\bm{\theta}}(\bm{x}^{(i)}))]_m
    \nonumber\\
    &\textstyle =-(1-\alpha)\mathrm{log} [\bm{f}_s(\bm{z}_{\bm{\theta}}(\bm{x}^{(i)}))]_t\nonumber\\
    &\textstyle ~~~~~-\frac{\alpha}{M-1}\sum_{m\neq t}\mathrm{log} [\bm{f}_s(\bm{z}_{\bm{\theta}}(\bm{x}^{(i)}))]_m\nonumber
\end{align}
since $p_t=1-\alpha$ and $p_m=\frac{\alpha}{M-1}$ for $m\neq t$.
By differentiating softmax cross-entropy, we obtain
\begin{align}\textstyle
\label{derLab}
\frac{\partial J}{\partial z_{\bm{\theta},k}}=\begin{cases}
[\bm{f}_s(\bm{z}_{\bm{\theta}^*}(\bm{x}^{(i)}))]_k+\alpha-1&k=t,\\\textstyle
[\bm{f}_s(\bm{z}_{\bm{\theta}^*}(\bm{x}^{(i)}))]_k-\frac{\alpha}{M-1}&\textstyle\mathrm{otherwise}.
\end{cases}
\end{align}
Since we assume that one of the elements of $\bm{z}_{\bm{\theta}^*}$ has a finite value, 
we have $\sum_m\mathrm{exp}z_{\bm{\theta}^*,m}>0$.
Thus, \req{derLab} for $k=t$ becomes 
\begin{align}\textstyle
[\bm{f}_s(\bm{z}_{\bm{\theta}^*}(\bm{x}^{(i)}))]_t&\textstyle=1-\alpha,
\label{prelabopt0}\\\textstyle
\frac{\mathrm{exp}z_{\bm{\theta}^*,t}}{\sum_m\mathrm{exp}z_{\bm{\theta}^*,m}}&\textstyle=1-\alpha,\nonumber\\\textstyle
\mathrm{exp}z_{\bm{\theta}^*,t}&\textstyle=(1-\alpha)(\sum_m\mathrm{exp}z_{\bm{\theta}^*,m}),\nonumber\\\textstyle
\alpha \mathrm{exp}z_{\bm{\theta}^*,t}&\textstyle=(1-\alpha)(\sum_{m\neq t}\mathrm{exp}z_{\bm{\theta}^*,m}),\nonumber\\\textstyle
z_{\bm{\theta}^*,t}&\textstyle=\log(\frac{1-\alpha}{\alpha}(\sum_{m\neq t}\mathrm{exp}z_{\bm{\theta}^*,m})).\label{labopt0}
\end{align}
In the same manner, we have 
\begin{align}\textstyle
&[\bm{f}_s(\bm{z}_{\bm{\theta}^*}(\bm{x}^{(i)}))]_k\textstyle=\frac{\alpha}{M-1},
\label{prelabopt}\\\textstyle
&z_{\bm{\theta}^*,k}\textstyle=\log(\frac{\alpha}{M-1-\alpha}\sum_{m\neq k}\mathrm{exp}(z_{\bm{\theta}^*,m}(\bm{x}^{(i)})))\label{labopt},
\end{align}
for $k\neq t$.
It is difficult to obtain the 
solutions of eqs.~(\ref{labopt0}) and (\ref{labopt}) in closed form, but we can show
they have finite values as follows.
If $z_{\bm{\theta}^*,k'}(\bm{x}^{(i)}))\rightarrow-\infty$ where $k'\neq t$,
\req{prelabopt} does not hold because the left side of \req{prelabopt} becomes 0 and 
$0<\alpha<M-1$.
If $z_{\bm{\theta}^*,t}(\bm{x}^{(i)}))\rightarrow\infty$,
\req{prelabopt0} does not hold because the left side of \req{prelabopt0} becomes 1.
Therefore, all elements of the logits have finite values.
\end{proof}
\begin{proposition}
    The optimal logits for logit squeezing
    $\!\bm{z}_{\bm{\theta}^*}(\bm{x}^{(i)})\!=\!\mathrm{arg}\!\min_{\bm{z}_{\bm{\theta}}} \mathcal{L}
    _{\mathrm{CE}}(\bm{z}_{\bm{\theta}}(\bm{x}^{(i)}),\bm{p}^{(i)}\!)\!+\!\frac{\lambda}{2}\!||\bm{z}_{\bm{\theta}}(\bm{x}^{(i)})||_2$ 
    satisfy
    \begin{align*}
    \!z_{\bm{\theta}^*\!,k}(\bm{x}^{(i)})\!=\!\begin{cases}
    (-[f_s(\bm{z}_{\bm{\theta}^*}(\bm{x}^{(i)})\!)]_t+1)/\lambda &k=t\\
    \!=\!-[f_s(\bm{z}_{\bm{\theta}^*}(\bm{x}^{(i)})\!)]_k/\lambda&k\neq t.\end{cases}
    \end{align*}
    Namely, all elements of the optimal logit vector $\bm{z}_{\bm{\theta}^*}(\bm{x}^{(i)})$ have finite values.
    \end{proposition}
\begin{proof}
Let $t$ be a label of data point $\bm{x}^{(i)}$.
The objective function of logit squeezing for $\bm{x}^{(i)}$ is 
$J=-\mathrm{log} [\bm{f}_s(\bm{z}_{\bm{\theta}}(\bm{x}^{(i)}))]_t+\frac{\lambda}{2}||\bm{z}_{\bm{\theta}}(\bm{x}^{(i)})||_2$.
Since we assume $\bm{z}_{\bm{\theta}}(\bm{x}^{(i)})$ can be an arbitrary vector,
$\left.\frac{\partial J}{\partial \bm{z}_{\bm{\theta}}}\right|_{\bm{z}_{\bm{\theta}}=\bm{z}_{\bm{\theta}^*}}=\bm{0}$ at the minimum points.
By differentiating $J$, we obtain 
\begin{align}\textstyle
\frac{\partial J}{\partial z_{\bm{\theta},k}}=
\begin{cases}
-1+[\bm{f}_s(\bm{z}_{\bm{\theta}}(\bm{x}^{(i)}))]_k+\lambda z_{\bm{\theta},k}&k=t,\\
[\bm{f}_s(\bm{z}_{\bm{\theta}}(\bm{x}^{(i)}))]_k+\lambda z_{\bm{\theta},k}&\mathrm{otherwise},
\end{cases}\nonumber
\end{align}
thus have
\begin{align}\textstyle
-1+[\bm{f}_s(\bm{z}_{\bm{\theta}^*}(\bm{x}^{(i)}))]_t+\lambda z_{\bm{\theta}^*,t}&=0,\nonumber\\\textstyle
z_{\bm{\theta}^*,t}&=\frac{1-[\bm{f}_s(\bm{z}_{\bm{\theta}^*}(\bm{x}^{(i)}))]_t}{\lambda}.\nonumber
\end{align}
Since each element of softmax functions is bounded as
 $0\leq [\bm{f}_s(\bm{z}_{\bm{\theta}^*}(\bm{x}^{(i)}))]_k\leq 1$,
 we have $0\leq z_{\bm{\theta}^*,t} \leq \frac{1}{\lambda}$.
In the same manner, we have $-\frac{1}{\lambda}\leq z_{\bm{\theta}^*,k} \leq 0$ for $k\neq t$.
Therefore, all elements of the optimal logits have finite values.
\end{proof}
\begin{theorem}
    % \label{mono}
    Let $g(z)$ be tanh or sigmoid function and $\gamma$ be a hyper-parameter
     satisfying $0\!<\!\gamma\!<\!\infty$.
    If we use tanh or sigmoid function before softmax
     as $\bm{f}_s(\gamma g(\bm{z}_{\bm{\theta}}(\bm{x}^{(i)})))$, 
    all elements of the optimal pre-logit vector
     $\bm{z}_{\bm{\theta}^*}(\bm{x}^{(i)})\!=\!\mathrm{arg}\min_{\bm{z}_{\bm{\theta}}} \mathcal{L}
    _{\mathrm{CE}}(\gamma g(\bm{z}_{\bm{\theta}}(\bm{x}^{(i)})),\bm{p}^{(i)}) $
    do not have finite values
    while all elements of the optimal logit vector $\gamma g(\bm{z}_{\bm{\theta}^*})$ has finite values.
    \end{theorem}
    \begin{proof}
First, we show the case using tanh as $g(z)=\mathrm{tanh}(z)$.
Let $t$ be a target label for $\bm{x}^{(i)}$. 
The objective function of softmax cross-entropy loss for $\bm{x}^{(i)}$ is
\begin{align}\textstyle
    \scalebox{0.95}{$
J=
\mathcal{L}_{\mathrm{CE}}(\gamma g(\bm{z}_{\bm{\theta}}(\bm{x}^{(i)})),\bm{p}^{(i)})
=-\mathrm{log} [\bm{f}_s(\gamma\mathrm{tanh}(\bm{z}_{\bm{\theta}}(\bm{x}^{(i)})))]_t.
$}\nonumber
\end{align}
Since we assume that $\bm{z}_{\bm{\theta}}(\bm{x}^{(i)})$ can be an arbitrary vector,
$\left.\frac{\partial J}{\partial \bm{z}_{\bm{\theta}}}\right|_{\bm{z}_{\bm{\theta}}^*(\bm{x}^{(i)})}=\bm{0}$ at the minimum point.
Since $\mathrm{tanh}$ is an element-wise function,
$\frac{\partial J}{\partial z_{\bm{\theta}^*,k}}$ is written as
\begin{align}\textstyle
\frac{\partial J}{\partial z_{\bm{\theta},k}}=\gamma\frac{\partial J}{\partial \gamma\mathrm{tanh}(z_{\bm{\theta},k})}\frac{\partial \mathrm{tanh}(z_{\bm{\theta},k})}{\partial z_{\bm{\theta},k}}.\nonumber
\end{align}
$\gamma\frac{\partial J}{\partial \gamma \mathrm{tanh}(z_{\bm{\theta},k})}$ is obtained as in \req{sceopteq} and does not become 0 since $-1\leq \mathrm{tanh}(z_{\bm{\theta},k})\leq 1$ and $0<\gamma<\infty$.
Thus, $\left.\frac{\partial J}{\partial z_{\bm{\theta},k}}\right|_{z_{\bm{\theta}^*,k}(\bm{x}^{(i)})}=0$ corresponds to $\left.\frac{\partial \mathrm{tanh}(z_{\bm{\theta},k})}{\partial z_{\bm{\theta},k}}\right|_{z_{\bm{\theta}^*,k}(\bm{x}^{(i)})}=0$.
We have 
\begin{align}\textstyle
\frac{\partial \mathrm{tanh}(z_{\bm{\theta},k})}{\partial z_{\bm{\theta},k}}=1-\mathrm{tanh}^2(z_{\bm{\theta},k}).\nonumber
\end{align}
Only if $z_{\bm{\theta}^*,k}(\bm{x}^{(i)})\rightarrow \pm \infty$, the following equation holds:
$\left.\frac{\partial \mathrm{tanh}(z_{\bm{\theta},k})}{\partial z_{\bm{\theta},k}}\right|_{z_{\bm{\theta}^*,k}(\bm{x}^{(i)})}=0$.
Therefore, all elements of the optimal input vector  $\bm{z}_{\bm{\theta}^*}$ do not have finite values.
The case using sigmoid function $\sigma (z)$ can be shown in the same manner.
Since the derivative of sigmoid becomes $\frac{\partial \sigma (z_{\bm{\theta},k})}{\partial z_{\bm{\theta},k}}=\sigma(z_{\bm{\theta},k})(1-\sigma(z_{\bm{\theta},k}))$, the equation
$\left.\frac{\partial \sigma(z_{\bm{\theta},k})}{\partial z_{\bm{\theta},k}}\right|_{z_{\bm{\theta}^*,k}(\bm{x}^{(i)})}=0$ holds when $z_{\bm{\theta}^*,k}(\bm{x}^{(i)})\rightarrow \pm \infty$.
\end{proof}
\begin{theorem}
    % \label{PropPro}
    Let $g(z)$ be BLF and $\gamma$ be a hyper-parameter satisfying $0\!<\!\gamma\!<\!\infty$.
    If we use $g(z)$ before softmax as $\bm{f}_s(\gamma g(\bm{z}_{\bm{\theta}}(\bm{x})))$, 
     all elements of the optimal pre-logit vector
     $\bm{z}_{\bm{\theta}^*}(\bm{x}^{(i)})\!=\!\mathrm{arg}\!\min_{\bm{z}_{\bm{\theta}}} \mathcal{L}
    _{\mathrm{CE}}(\gamma g(\bm{z}_{\bm{\theta}}(\bm{x}^{(i)})),\bm{p}^{(i)}) $
    have finite values, and all elements of the optimal logit vector
     $\gamma g(\bm{z}_{\bm{\theta}^*})$ also have finite values.
    Specifically, we have the following equalities and inequalities:
    \begin{align}\textstyle
    \gamma g(z_{\bm{\theta}^*,k}(\bm{x}^{(i)}))&\textstyle=
    \begin{cases}
    \gamma\max_z g(z)~~~k=t\\\textstyle
    \gamma\min_z g(z)~~~\mathrm{otherwise},
    \end{cases}\\
    \gamma<&|\gamma g(z_{\bm{\theta}^*,k}(\bm{x}^{(i)}))|<\gamma\frac{\sqrt{5}+1}{2},\nonumber\\
    z_{\theta^*,k}(\bm{x}^{(i)})&\textstyle=\begin{cases}\mathrm{arg}\!\max_z g(z)~~~k=t\\\textstyle
    \mathrm{arg}\!\min_z g(z)~~~\mathrm{otherwise},
    \end{cases}\\
    2<&|z_{\bm{\theta}^*,k }(\bm{x}^{(i)})|<\sqrt{5}+1.\nonumber
    \end{align}
    \end{theorem}
    \begin{proof}
Let $t$ be a target label for $\bm{x}^{(i)}$. 
The objective function of softmax cross-entropy loss for $\bm{x}^{(i)}$ is
\begin{align}
J=
\mathcal{L}_{\mathrm{CE}}(\gamma g(\bm{z}_{\bm{\theta}}(\bm{x}^{(i)})),\bm{p}^{(i)})
=-\mathrm{log} [\bm{f}_s(\gamma g(\bm{z}_{\bm{\theta}}(\bm{x}^{(i)})))]_t,\nonumber
\end{align}
where $g(z)=2\left\{z\sigma(z)+\sigma(z)-z\sigma^2(z)\right\}-1$.
Since we assume that $\bm{z}_{\bm{\theta}}(\bm{x}^{(i)})$ can be an arbitrary vector,
$\left.\frac{\partial J}{\partial \bm{z}_{\bm{\theta}}}\right|_{\bm{z}_{\bm{\theta}^*}(\bm{x}^{(i)})}=\bm{0}$ at the minimum point.
Similary to tanh, BLF $g(z)$ is an element-wise function and
bounded by finite values.
Therefore, $\gamma\frac{\partial J}{\partial \gamma g(z_{\bm{\theta},k})}$ does not become 0.
As a result, $\left.\frac{\partial J}{\partial \bm{z}_{\bm{\theta}}}\right|_{\bm{z}_{\bm{\theta}}^*(\bm{x}^{(i)})}=\bm{0}$ holds when $\left.\frac{\partial g(z_{\bm{\theta},k})}{\partial z_{\bm{\theta},k}}\right|_{z_{\bm{\theta},k}^*(\bm{x}^{(i)})}=0$ for all $k$.
We have     
\begin{align}
    \scalebox{0.95}{$
        \frac{\partial g(z_{\bm{\theta},k})}{\partial z_{\bm{\theta},k}}=2\sigma(z_{\bm{\theta},k})(1-\sigma(z_{\bm{\theta},k}))(2+z_{\bm{\theta},k}-2z_{\bm{\theta},k}\sigma(z_{\bm{\theta},k})),
        $}\nonumber
    \end{align}
thus, candidates of the optimal points satisfy one of the following conditions:
 $\sigma(z)=0$, $\sigma(z)=1$, and $2+z-2z\sigma(z)=0$.
 Inputs that satisfy $\sigma(z)=0$ and $\sigma(z)=1$ correspond to
$z\rightarrow-\infty$ and $z\rightarrow\infty$,
respectively. Their outputs $g(z)$ become $\lim_{z\rightarrow-\infty}g(z)=-1$ and $\lim_{z\rightarrow\infty}g(z)=1$, respectively.
To investigate points satisfying $2+z-2z\sigma(z)=0$, we define $f(z)=2+z-2z\sigma(z)$,
which is a continuous function.
This function can be written as
\begin{align}\textstyle
f(z)&\textstyle=2+z-2z\sigma(z),\nonumber\\
&\textstyle=2+z-2z\frac{1}{1+e^{-z}},\nonumber\\
&\textstyle=2+\frac{z(1+e^{-z})-2z}{1+e^{-z}},\nonumber\\
&\textstyle=2-z\frac{1-e^{-z}}{1+e^{-z}},\label{interEq}\\
&\textstyle=2-z\mathrm{tanh}(\frac{z}{2}).\nonumber
\end{align}
Since $\mathrm{tanh(z)}<z$ for $z>0$ and $\mathrm{tanh(z)}>z$ for $z<0$, we have 
 \begin{align}\textstyle
 f(z)=2-z\mathrm{tanh}(z/2)>2-\frac{z^2}{2}~~\mathrm{for}~~z\neq 0.\nonumber
 \end{align}
By using this equation, we have $f(z)>0$ for $-2<z<2$.
In addition, since $\frac{z}{1+z}<1-e^{-z}$ for $z>-1$, we have
the following inequality from \req{interEq}:
\begin{align}\textstyle
f(z)&\textstyle=2-z\frac{1-e^{-z}}{1+e^{-z}},\nonumber\\
&\textstyle=\frac{2(1+e^{-z})-z(1-e^{-z})}{1+e^{-z}},\label{interEq2}\\
&\textstyle=\frac{-(2+z)(1-e^{-z})+4}{1+e^{-z}},\nonumber\\
&\textstyle<\frac{-(2+z)z+4(z+1)}{(1+e^{-z})(z+1)},\nonumber\\
&\textstyle=\frac{-(z-1)^2+5}{(1+e^{-z})(z+1)}.\nonumber
\end{align}
Thus, we have $f(z)<0$ for $z>\sqrt{5}+1$.
On the other hand, since we have $e^z<\frac{1}{1-z}$ for $z<1$,
we have the following inequality from \req{interEq2}:
\begin{align}\textstyle
f(z)&\textstyle=\frac{2(1+e^{-z})-z(1-e^{-z})}{1+e^{-z}},\nonumber\\
&\textstyle=\frac{2(e^z+1)-z(e^z-1)}{e^z+1},\nonumber\\
&\textstyle=\frac{e^z(2-z)+2+z}{e^z+1},\nonumber\\
&\textstyle<\frac{2-z+(2+z)(1-z)}{(e^z+1)(1-z)},\nonumber\\
&\textstyle=\frac{-(z+1)^2+5}{(1+e^{-z})(1-z)}.\nonumber
\end{align}
Thus, we have $f(z)<0$ for $z<-\sqrt{5}-1$.
Therefore, the points satisfying $\frac{\partial g(z)}{\partial z}=0$ are included in
$-\sqrt{5}-1<z<-2$ and $2<z<\sqrt{5}+1$ from intermediate value theorem.
The $g(z)$ can be written using $f(z)$ as
\begin{align}
g(z)&=2\left\{z\sigma(z)+\sigma(z)-z\sigma^2(z)\right\}-1,\nonumber\\
&=2\sigma(z)\left\{\frac{z}{2}+\frac{f(z)}{2}\right\}-1.\nonumber
\end{align}
Let $\underline{z}^*$ be inputs satisfying $f(\underline{z}^*)=0$ and $\sqrt{5}-1<\underline{z}^*<-2$,
and  $\bar{z}^*$ be inputs satisfying $f(\bar{z}^*)=0$ and $2<\bar{z}^*<\sqrt{5}+1$.
We have
\begin{align}\textstyle
g(\underline{z}^*)&\textstyle=\sigma(\underline{z}^*)\underline{z}^*-1=1+\frac{\underline{z}^*}{2}-1=\frac{\underline{z}^*}{2},\nonumber\\
\textstyle
g(\bar{z}^*)&\textstyle=\sigma(\bar{z}^*)\bar{z}^*-1=1+\frac{\bar{z}^*}{2}-1=\frac{\bar{z}^*}{2},\nonumber
\end{align}
where we use $f(\underline{z}^*)=f(\bar{z}^*)=2+z^*-2z^*\sigma(z^*)=0$.
Thus, we have $-\frac{\sqrt{5}+1}{2}<g(\underline{z}^*)<-1$ and $1<g(\bar{z}^*)<\frac{\sqrt{5}+1}{2}$.
Since $g(\underline{z}^*)$ is lesser than $\lim_{z\rightarrow -\infty}g(z)=-1$ and $g(\bar{z}^*)$ is greater than $
\lim_{z\rightarrow +\infty}g(z)=1$, $\underline{z}^*$ is the minimum point $\underline{z}^*=\mathrm{arg}\min_z g(z)$
 and 
$\bar{z}^*$ is the maximum point $\bar{z}^*=\mathrm{arg}\max_z g(z)$. 
Therefore, the optimal inputs $z_{\bm{\theta}^*,k}(\bm{x}^{(i)})$ and logits $g(z_{\bm{\theta}^*,k}(\bm{x}^{(i)}))$ are finite values for BLF, and
each element of the optimal inputs $z_{\bm{\theta}^*,k}(\bm{x}^{(i)})$ is 
$\underline{z}^*$ or $\bar{z}^*$.
The objective function 
$J=-\mathrm{log} [\bm{f}_s(\gamma g(\bm{z}_{\bm{\theta}}(\bm{x}^{(i)})))]_t=
-\mathrm{log} \frac{\mathrm{exp}(\gamma g(z_{\bm{\theta},t}(\bm{x}^{(i)})))}
{\sum_{m=1}^{M} \mathrm{exp}(\gamma g(z_{\bm{\theta},m}(\bm{x}^{(i)})))}$ 
becomes the smallest when
\begin{align}
g(z_{\bm{\theta},k}(\bm{x}^{(i)}))=
\begin{cases}\textstyle
g(\bar{z}^*) &\textstyle
k=t,\\\textstyle
g(\underline{z}^*) &\textstyle
\mathrm{otherwise}.
\end{cases}\nonumber
\end{align}
Therefore, the optimal logits are
\begin{align}
\textstyle
\gamma g(z_{\bm{\theta}^*,k}(\bm{x}^{(i)}))&\textstyle=
\begin{cases}
\gamma\max_z g(z)~~~k=t,\\\textstyle
\gamma\min_z g(z)~~~\mathrm{otherwise},
\end{cases}\nonumber
\end{align}
and the optimal inputs are
\begin{align}\textstyle
z_{\bm{\theta}^*,k}(\bm{x}^{(i)})&\textstyle=
\begin{cases}\mathrm{arg}\max_z g(z)~~~k=t,\\\textstyle
\mathrm{arg}\min_z g(z)~~~\mathrm{otherwise},
\end{cases}\nonumber
\end{align}
Thus, we have
\begin{align}\textstyle
\gamma&<|\gamma g(z_{\bm{\theta}^*,k}(\bm{x}^{(i)}))|<\textstyle
\gamma\frac{\sqrt{5}+1}{2},\nonumber\\2&<|z_{\bm{\theta}^*,k}(\bm{x}^{(i)})|<\sqrt{5}+1.\nonumber
\end{align}
\end{proof} 
We developed BLF inspired by the derivative of swish \cite{swish} since it has finite maximum and minimum points and is a continuous function.
\section{Detailed experimental conditions}
\label{excond}
To generate PGD and SPSA attacks, we used advertorch \cite{ding2018advertorch}.
Since previous studies did not split the training datasets of MNIST and CIFAR10 into validation sets and training sets \cite{TRADES,pgd2},
we did not split them.
Our codes for experiments are based on the open-source code \cite{ExCode}
\subsection{Conditions for empirical evaluation of logit regularization in Section 3.2}
In this experiment, the model architecture was ResNet-18 (RN18).
For tanh and BLF, we added these activation function before softmax.
We trained all models by using SGD with momentum of 0.9 for 350 epochs and
used weight decay of 0.0005.
The initial learning rate was set to 0.1.
After the 150-th epoch and the 250-th epoch, we divided the learning rate by 10.
The minibatch size was set to 128.
We used models at the epoch when they achieved the largest clean accuracy.
To make average logit norms obtain various values, we use various hyperparameters for each
logit regularization method.
For label smoothing, we set $\alpha$ to [0.005, 0.01, 0.05, 0.1, 0.3, 0.5, 0.75, 0.85];
for logit squeezing, we set $\lambda$ to [0.005, 0.01, 0.05, 0.1, 0.3, 0.5, 0.75, 0.9];
for tanh, we set $\gamma$ to [0.1, 0.2, 0.3, 0.4, 0.5, 0.8, 1.0, 1.2];
and for BLF, we set $\gamma$ to [0.1, 0.2, 0.3, 0.4, 0.5, 0.8, 1.0, 1.2].
After standard training of the models, we evaluate their accuracies on the test data of CIFAR-10 attacked by
PGD.
We set the step size of PGD to 2/255 and 
the number of iteration to 100. The $L_\infty$ norm of perturbations was set to 4/255.
\subsection{Experimental conditions for experiments discussed in Section 5}
Our experimental conditions depend on the model architectures.
\subsubsection*{Experimental conditions for small CNNs}
For MNIST, we used a CNN composed of two convolutional layers and two fully connected layers 
(2C2F) and one composed of four convolutional layers 
and three fully connected layers (4C3F).
The details of these model architectures are as follows:
\begin{description}
\item{\textbf{2C2F} }
 The first convolutional layer had the 10 output channels and the second convolutional
layer had 20 output channels.
The kernel sizes of the convolutional layers were set to 5, and 
their strides were set to 1. We did not use zero-padding in these layers.
We added max pooling with the stride of 2 and ReLU activation after each convolutional layer.
The size of the first fully connected layer was set to $320\times 50$, and we used the ReLU activation 
after this layer.
The size of the second fully connected layer was set to $50\times 10$, and we used softmax as the 
output function. 
Between the first and the second convolution layer and between the first and the second fully 
connected layer, we applied 50~\% dropout.
We used the default initialization of PyTorch 1.0.0 to initialize all parameters.
\item{\textbf{4C3F}} 
We used an implementation of this model architecture released by the authors of \cite{TRADES}.
The first and second convolutional layers had 32 output channels, and 
the third and fourth convolutional layer had 64 output channels.
The kernel sizes of the convolutional layers were set to 3, and 
their strides were set to 1. We did not use zero-padding in these layers.
We used the ReLU activation after each convolution layer and
 added max pooling with the stride of 2 after the second and fourth ReLU activation.
Three fully connected layers followed these convolution layers.
The size of the first fully connected layer was set to $1024\times 200$, and that of 
the size of the second fully connected layer was set to $200\times 200$.
The size of the last fully connected layer was set to $200\times 10$.
After the first and the second fully connected layers, we added ReLU activation
and applied 50~\% dropout into the output of the first fully connected layers.
We used softmax as the output function. 
We used the default initialization of PyTorch 1.0.0 to initialize all parameters except for 
biases and weights of the last fully connected layer.
All these biases and weights were initialized as zero.
\end{description}
We trained models by using SGD with momentum of 0.5 and learning rate of 0.01 for 100 epochs.
For 2C2F,
we used weight decay of 0.01 in both the standard training and adversarial training settings. 
For 4C3F, the learning rate SGD was divided by 10 after the 55-th, 75-th, and 90-th epoch. 
We did not use weight decay for training of 4C3F
in both the standard training and adversarial training settings . 
The minibatch size was set to 64.
We used the models at the epoch when they achieved the largest clean accuracy.
We evaluate the following hyper-parameter sets: $\lambda$ of [0.01, 0.05, 0.1, 0.3, 0.5]
for logit squeezing, $\alpha$ of [0.05, 0.1, 0.3, 0.5, 0.7] for label smoothing,
$\gamma$ of [0.5, 0.8, 1.0, 1.5, 3.5] for BLF , and $\beta$ of [3, 6, 9, 12, 15] for TRADES,
and we select hyperparameters that achieve  the largest clean accuracy for standard training,
and those that achieve the largest adversarial accuracy against
PGD ($\varepsilon=0.3$) for adversarial training.
We initialized $\tilde{\gamma}$ as -1 in L-BLF.
\subsubsection*{Experimental conditions for RN18}
We trained models by using SGD with momentum of 0.9 for 350 epochs in the standard training setting
and for 120 epochs in the adversarial training setting.
In both standard and adversarial training settings, we used weight decay of 0.0005.
The initial learning rate was set to 0.1.
After the 150-th epoch and the 250-th epoch, we divided the learning rate by 10 in the standard 
training setting.
In the adversarial training setting, we divided the learning rate by 10 after the 50-th epoch and 100-th epoch.
The minibatch size was set to 128.
We used the default initialization of PyTorch 1.0.0 to initialize all parameters except for $\tilde{\gamma}$.
We initialized $\tilde{\gamma}$ as -1 in L-BLF.
We used the models at the epoch when they achieved the largest clean accuracy.
We evaluate the following hyper-parameter sets: $\lambda$ of [0.01, 0.05, 0.1, 0.3, 0.5] for logit 
squeezing, $\alpha$ of [0.05, 0.1, 0.3, 0.5, 0.7] for label smoothing,
$\gamma$ of [0.1, 0.5, 0.8, 1.0, 1.2] for BLF, and $\beta$ of [3, 6, 9, 12, 15] for TRADES, 
and we selected hyperparameters that achieve the largest clean accuracy for the standard training setting
and those that achieves the largest adversarial accuracy against
PGD ($\varepsilon=8/255$) for the adversarial training setting.
\subsubsection*{Experimental conditions for WRN}
We used an implementation of WRN released by the authors of \cite{TRADES}.
We trained models by using SGD with momentum of 0.9 for 350 epochs in the standard training setting
and for 120 epochs in the adversarial training  setting.
In both standard and adversarial training settings, we used weight decay of 0.0005.
The initial learning rate was set to 0.1.
After the 150-th epoch and the 250-th epoch, we divided the learning rate by 10 in the standard 
training setting.
In the adversarial training setting, we divided learning rate by 10 after the 75-th epoch and 100-th epoch 
following \cite{TRADES}\footnote{This learning rate schedule is based on the code of \cite{TRADES}: https://github.com/yaodongyu/TRADES}.
We used the default initialization of PyTorch 1.0.0 to initialize all parameters except $\tilde{\gamma}$.
We initialized $\tilde{\gamma}$ as -1 in L-BLF.
We used models at the epoch when models achieved the largest clean accuracy.
We used the same hyper-parameters $\alpha$, $\lambda$, $\gamma$, and $\beta$ used in the experiments of RN18.
\section{Limitation of logit regularization methods}
\label{LimSec}
\citet{engstrom2018evaluating} have shown that adversarial logit pairing, which is similar to logit regularization,
is sensitive to targeted attacks, and 
\citet{mosbach2018logit} have shown that logit squeezing is sensitive to PGD attacks with multi-restart.
In this section, we evaluate logit regularization methods with strong untargeted and targeted PGD attacks with multi-start.
As strong PGD attacks, we used AutoPGD (APGD) with cross-entropy and targeted AutoPGD with difference of logits Ratio Loss (TAPGD) in \cite{AutoAttack}.
APGD is more sophisticated than PGD; APGD uses momentum and adaptively selects the step size. We restarted APGD and TAPGD five times.
In the experiments, we tuned hyperparameters ($\lambda$, $\alpha$, $\gamma$, $\beta$) for each attack
and train the model for one time for each hyper-parameter.

Results are listed in \rtab{APGDtab}.
In this table, logit regularization methods are superior or comparable to Baseline.
In particular, logit regularization methods outperform TRADES on MNIST.
On CIFAR10, robust accuracies against TAPGD become zero when standard training.
However, when using adversarial training, logit regularization methods can outperform Baseline;
i.e., logit regularization contributes to general adversarial robustness.
Robust accuracies against APGD of BLF become the highest on a majority of settings.
Thus, BLF is effective in defending against untargeted strong attacks.

Indeed, logit regularization methods without adversarial training are not robust enough for strong attacks.
However, they are still effective against practical threat models, e.g., untargeted attacks or black box attacks.
Furthermore, when combining adversarial training with these methods, they can be comparable to strong defense methods.
\begin{table}[tbp]
    \centering
    \caption{Robust accuracies against APGD attacks on MNIST ($\varepsilon\!=\!0.15$)  and CIFAR10 ($\varepsilon\!=\!8/255$).
    Logit regularization methods are weak to targeted attacks in the standard training setting, especially on CIFAR10.}
    \label{APGDtab}
    \resizebox{\columnwidth}{!}{
    \begin{tabular}{ccccccc}\toprule
    &Baseline&LSQ&LSM&BLF&TRADES\\\midrule
    APGD (2C2F, ST)&    43.8 &     58.9 &      $\bm{69.2}$ &  60.5&   N/A \\
    APGD (2C2F, AT) &    91.8 &     90.9 &      91.4 &  $\bm{92.3}$ &  86.8 \\
    TAPGD (2C2F,ST) &    44.2 &     49.7 &       $\bm{61.70 }$&     57.0 &   N/A \\
    TAPGD (2C2F,AT) &     91.9 &     90.7 &      91.2 &  $\bm{92.4}$ &  85.2 \\\midrule
    APGD (4C3F, ST)        &    31.6 &      41.8 &      48.5 &  $\bm{75.8}$ &   N/A \\
    APGD (4C3F,AT)        &    98.0 &     $\bm{98.2}$ &      98.0 &   98.1 &  98.0 \\
    TAPGD (4C3F,ST) &    27.7 &     $\bm{39.8}$ &      37.9 &  38.6 &   N/A \\
    TAPGD (4C3F,AT) &    98.1 &     $\bm{98.3}$ &      98.1 &  98.2 &   98.1 \\\midrule\midrule
    APGD (RN18, ST)        &        0.0 &      0.3 &       1.4 &   $\bm{3.8 }$&   N/A \\
    APGD (RN18, AT)        &    48.6 &     50.8 &      50.5 &  51.5 &   $\bm{52.6 }$\\
    TAPGD (RN18, ST) &        0.0 &         0.0 &          0.0 &      0.0 &   N/A \\
    TAPGD (RN18, AT) &    46.1 &     47.4 &      47.6 &  46.2 &  $\bm{49.9}$ \\\midrule
    APGD (WRN, ST)        &      0.0 &      0.0 &       0.8 &   $\bm{6.0 }$&   N/A \\
    APGD (WRN, AT)        &    53.1 &     53.0 &      54.0 &  $\bm{55.9}$ &  $\bm{55.9}$ \\
    TAPGD (WRN, ST) &      0.0 &       0.0 &        0.0 &    0.0 &   N/A \\
    TAPGD (WRN, AT) &    51.6 &     51.4 &      52.0 &  51.0 &  $\bm{53.8}$ \\
    \bottomrule
    \end{tabular}
    }
\end{table}
\section{Comparison with other bounded functions}
\label{EvalSec}
\begin{figure}[tb]
    \centering
    \subfloat[sine wave]{\includegraphics[width=.78\linewidth]{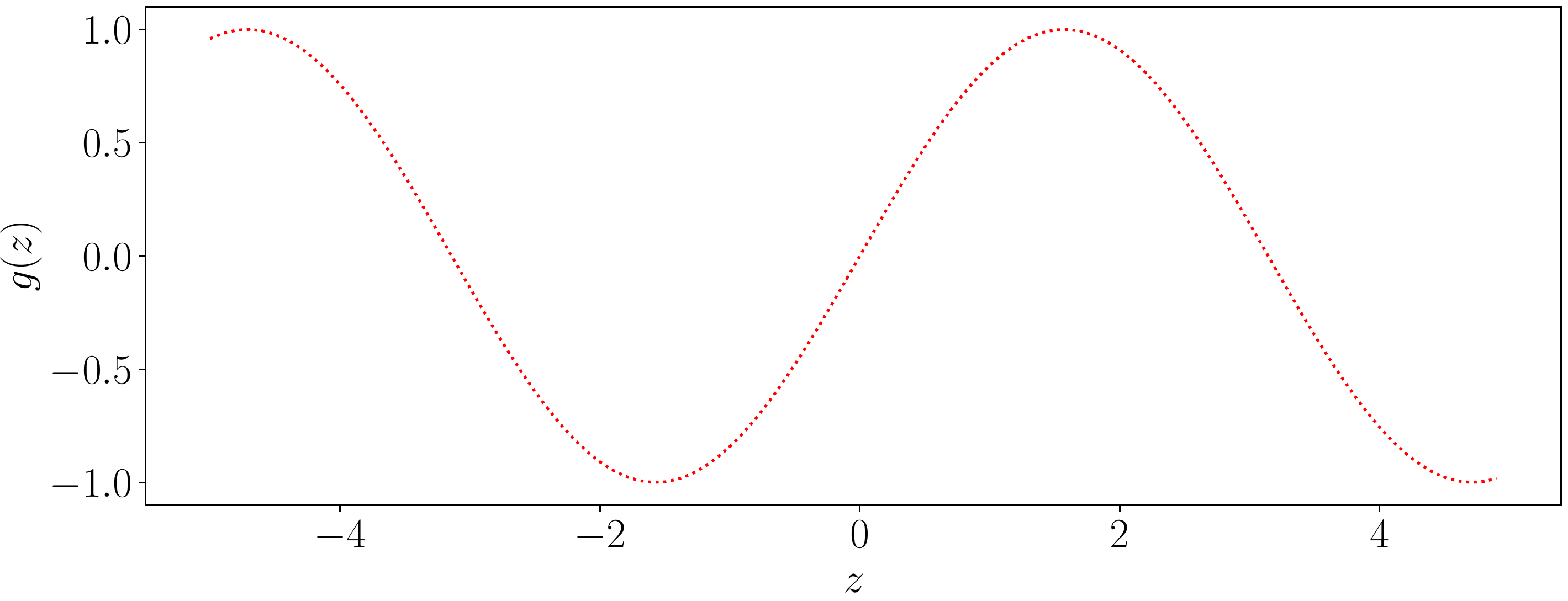}}\\
    \subfloat[single wave]{\includegraphics[width=.78\linewidth]{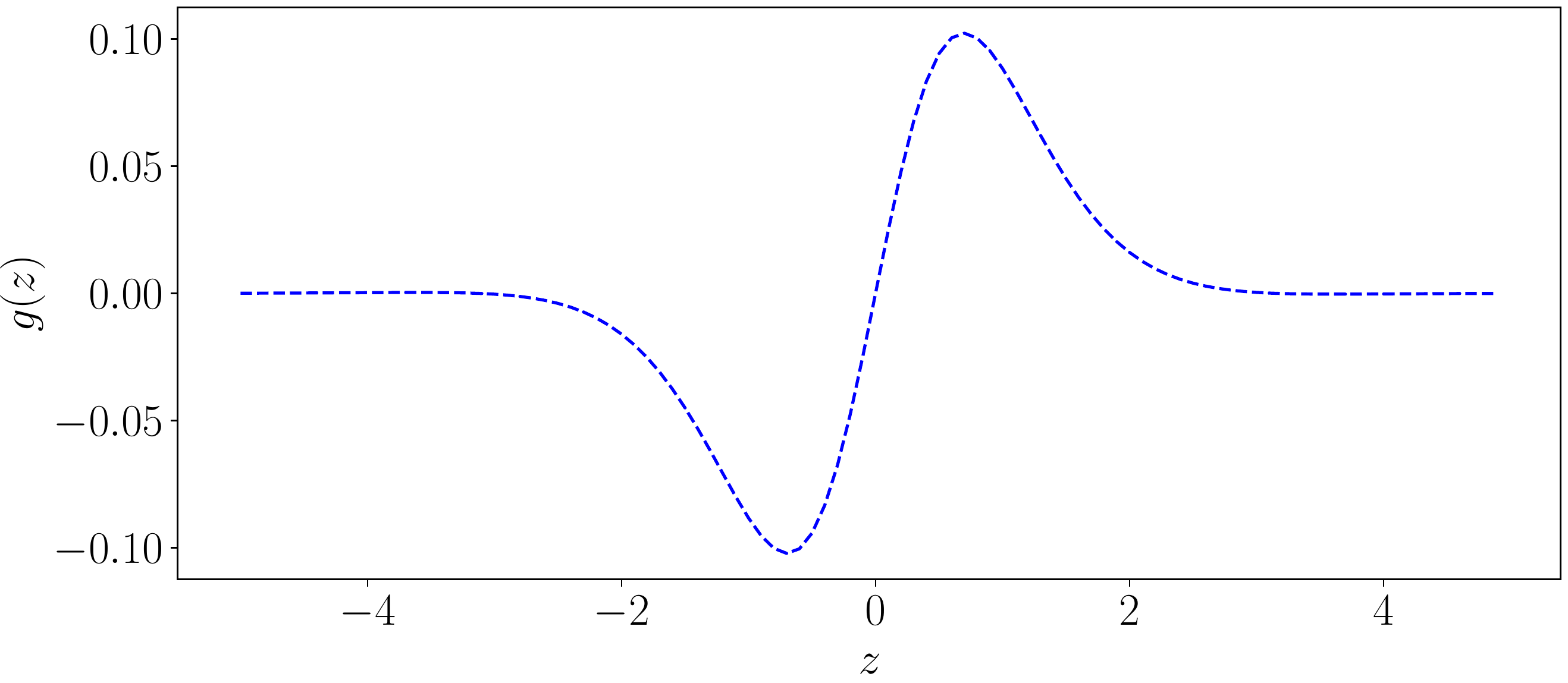}}\\
    \caption{Other bounded functions}
    \label{Fun}
    \end{figure}
\begin{figure}[tb]
    \centering
    \subfloat[MNIST Standard training]{\includegraphics[width=.49\linewidth]{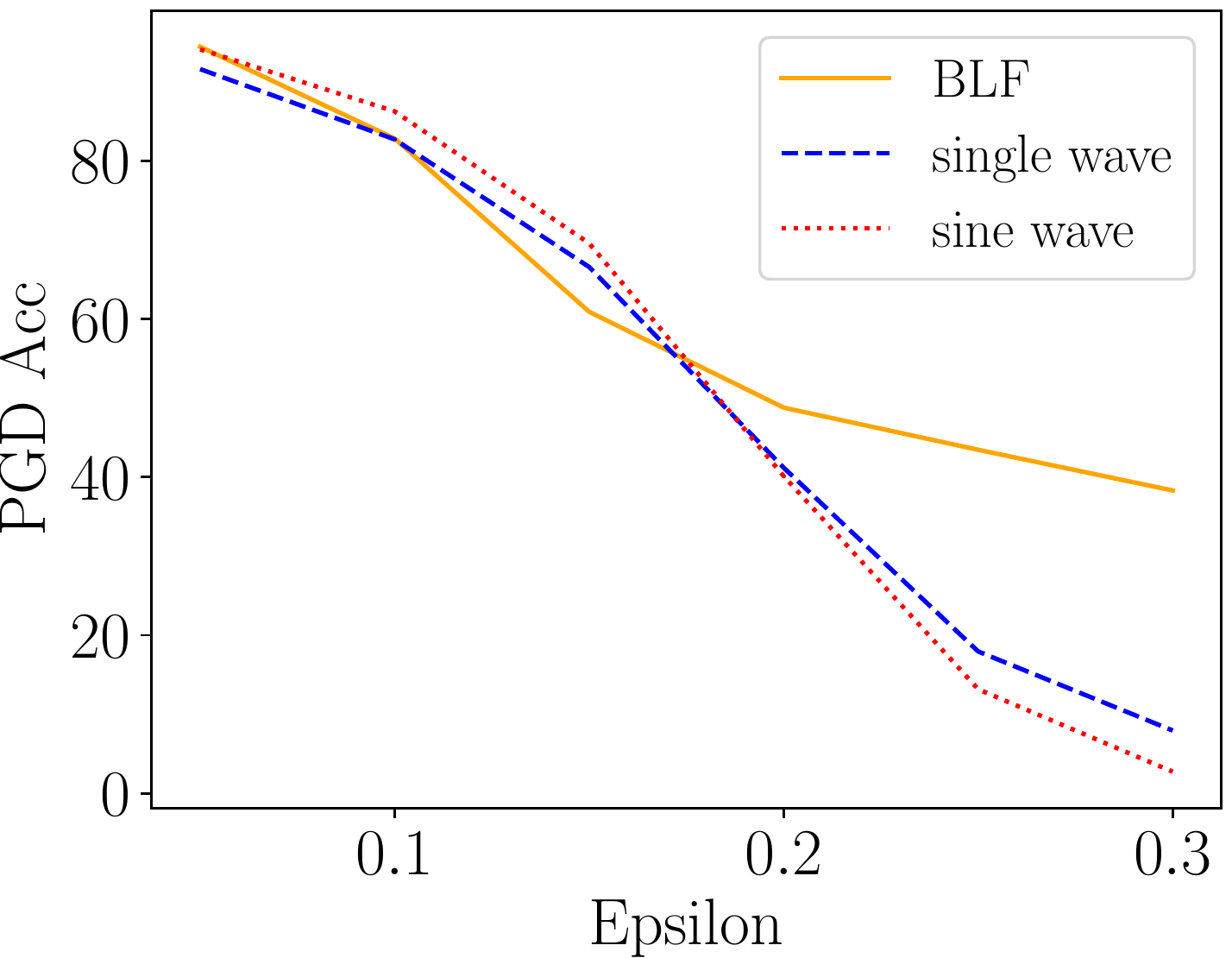}}
    \centering\hfill
    \subfloat[MNIST Adversarial training]{\includegraphics[width=.49\linewidth]{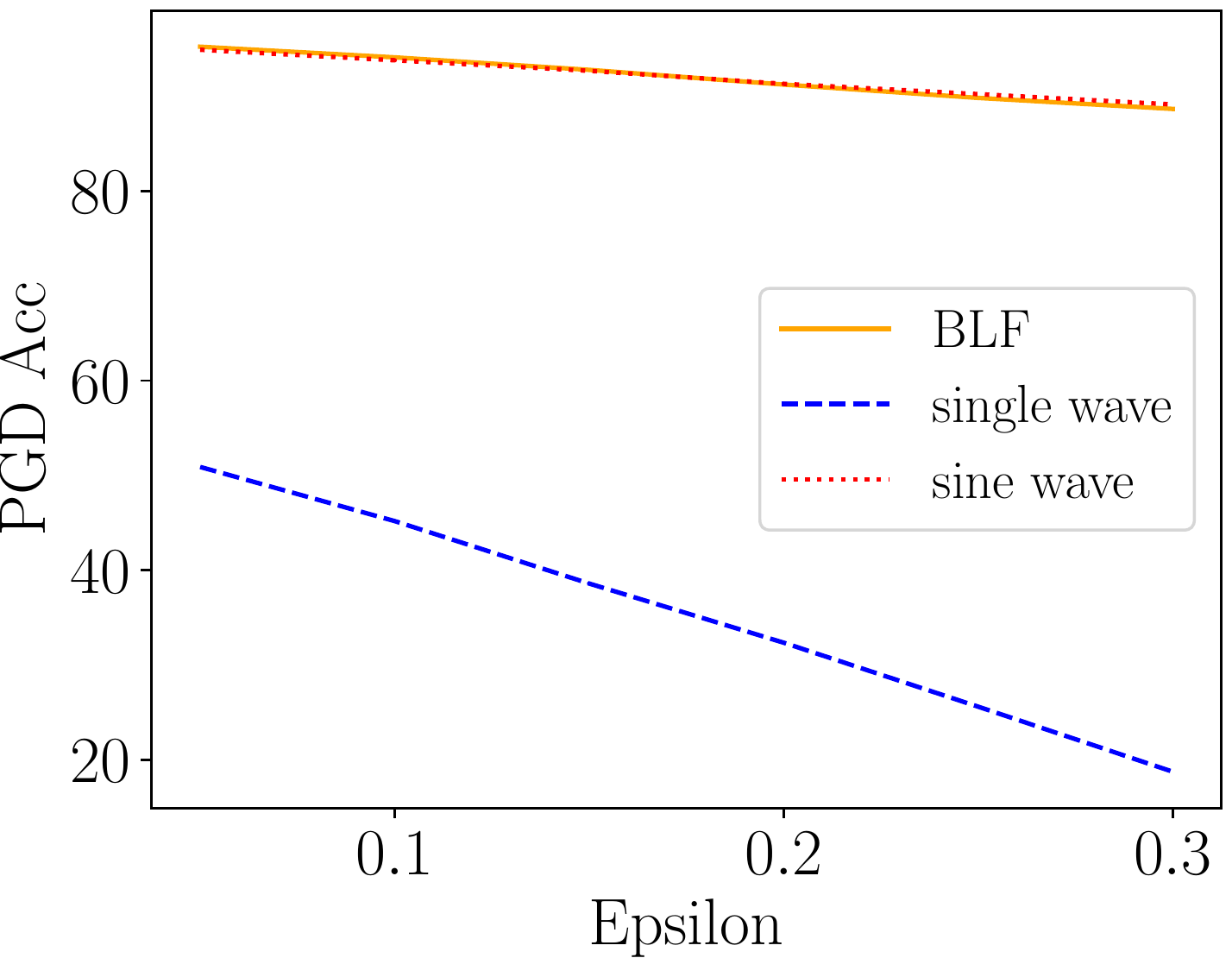}}\\
    \centering
    \subfloat[CIFAR10 Standard training]{\includegraphics[width=.49\linewidth]{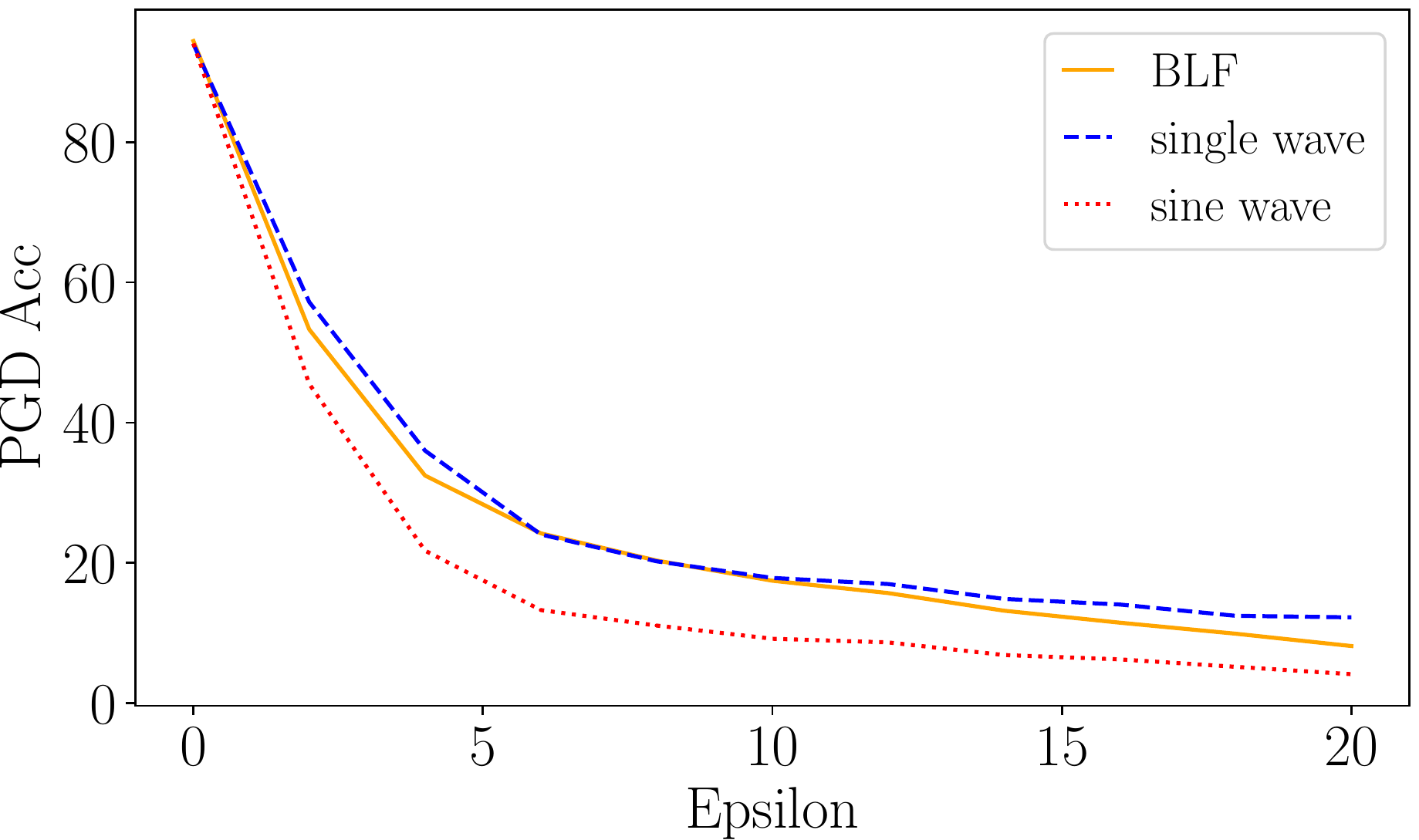}}
    \centering\hfill
    \subfloat[CIFAR10 Adversarial training]{\includegraphics[width=.49\linewidth]{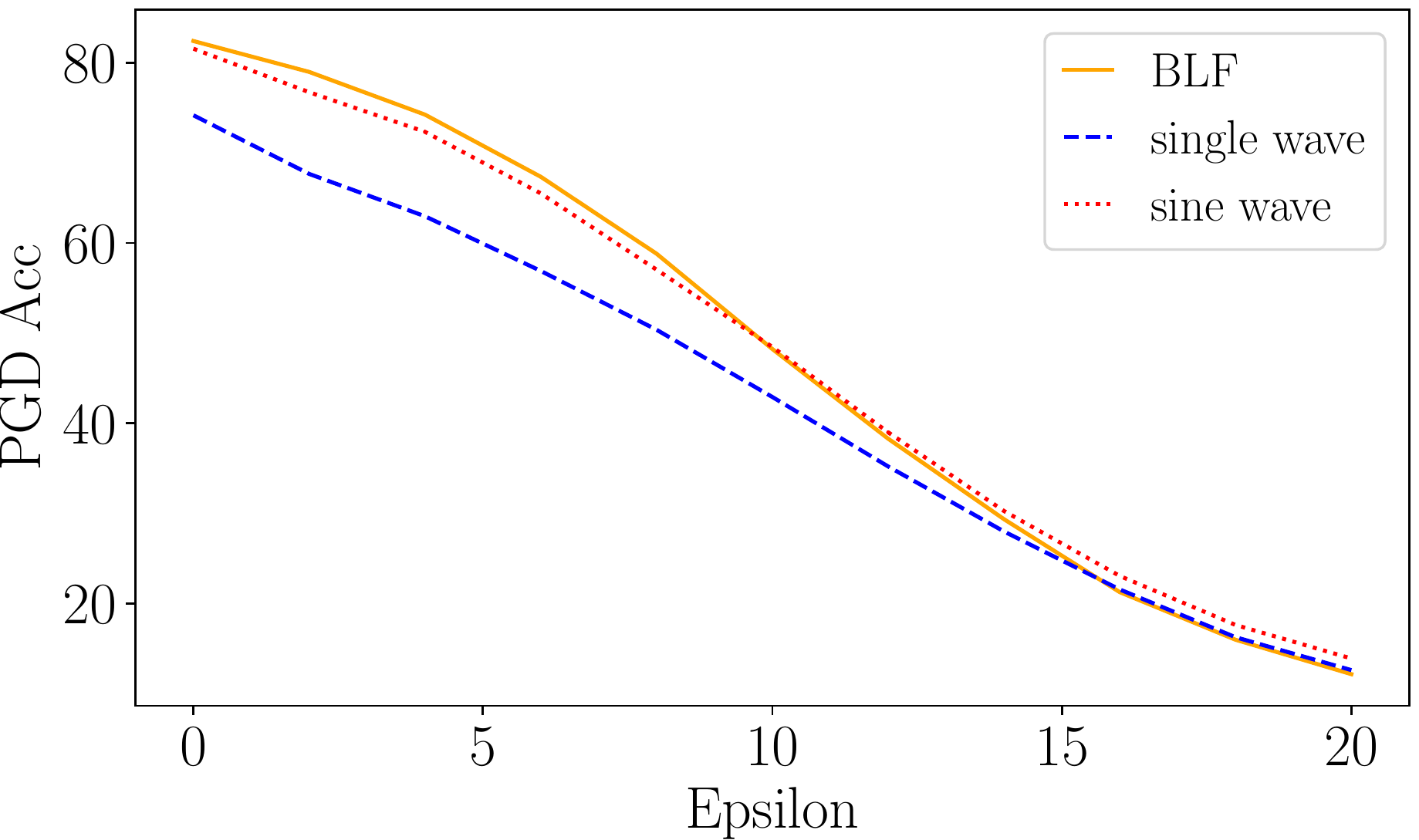}}
    \caption{Robust accuracy of 2C2F on MNIST and RN18 on CIFAR10.}
    \label{OBLF}
    \end{figure}
In this section, we compare BLF with other bounded functions.
As bounded functions, we evaluated sine wave $z=\sin(z)$ (\rfig{Fun} (a)) and single wave $z=\sin(z)/(\mathrm{exp}(z)+\mathrm{exp}(-z))^2$ (\rfig{Fun} (b)).
We used 2C2F on MNIST and RN18 on CIFAR10 and the same experimental conditions in Section~5.2.
Robust accuracies against PGD attacks are shown in \rfig{OBLF}.
This figure shows that single wave is inferior to BLF, and sine wave is comparable to BLF.
Note that our proposal is using bounded functions, which have finite maximum and minimum points,
 and not limited to using BLF.
 We use BLF because BLF is similar to tanh and we can evaluate the effect of finite optimal points.
 Even so, BLF can achieve better performance than single wave and as good performance as sine wave.
\section{Evaluation of input loss surface}
Since \citet{mosbach2018logit} have pointed out that
logit regularization methods cause distorted input loss surfaces,
we visualized input loss surface of each method.
In this experiment, we randomly selected eight data points $\{\bm{x}^(i)\}_{i=1}^8$ from the training dataset of CIFAR10
and generated two noise vectors $\bm{v}_1, \bm{v}_2$ whose elements are randomly selected from \{-1, +1\}.
Then, we evaluated $\mathcal{L}_{CE}(\bm{x}+\varepsilon_1\bm{v}_1+\varepsilon_2\bm{v}_2)$ for each dat point.
$\varepsilon_1$ and $\varepsilon_2$ are noise levels, and we changed $\varepsilon_1$ and $\varepsilon_2$ from -16/255 to 16/255, in 0.5/255 increments.
Note that we used the same data points and noise vectors among Baseline, logit regularization methods and TRADES.
In this experiment, we used RN18 trained on CIFAR10 in the standard training and adversarial training settings.

Figures~\ref{NLoss}-\ref{NBLFLoss} are input loss surfaces of models trained in the
standard setting, and 
Figures~\ref{PLoss}-\ref{PTRADESLoss} are input loss surfaces of models trained in the
adversarial setting.
In the standard training setting, when we compare Baseline (\rfig{NLoss}) with
logit regularization methods (Figs.~\ref{NLSQLoss}-\ref{NBLFLoss}),
logit regularization methods seem to make the flat input space small (e.g., Data No.2, No.5, and No.6).
However, we should pay attention to the scale of loss surfaces.
Scales of loss surfaces for logit regularization methods are smaller than those for Baseline.
For example, losses of Data No.2, No.5, and No.6 for Baseline
can exceed ten by noise injection while losses for logit regularization methods are smaller than three.
Thus, logit regularization methods improve the robustness against random noise.
This is because logit regularization methods induce the small Lipschitz constants.
Similarly, in the adversarial training settings, logit regularization methods
make the amount of loss changes small, e.g., the amount of change of the loss of No.2 in \rfig{PBLFLoss}
is smaller than the loss scale of No.2 in \rfig{PLoss}.
We list the difference between the maximum loss and minimum loss ($\max_{\varepsilon_1, \varepsilon_2} \mathcal{L}_{CE}(\bm{x}+\varepsilon_1\bm{v}_1+\varepsilon_2\bm{v}_2)-\min_{\varepsilon_1, \varepsilon_2} \mathcal{L}_{CE}(\bm{x}+\varepsilon_1\bm{v}_1+\varepsilon_2\bm{v}_2)$)
in \rtab{MMdiff}.
We can see that BLF can make the difference of losses by noise injection the smallest,
and thus, BLF can improve the robustness of cross entropy loss against the random noise injection.
These results support that logit regularization methods can improve the robustness against
untargeted attack using cross entropy.
\begin{figure*}
    \centering
    \includegraphics[width=\linewidth]{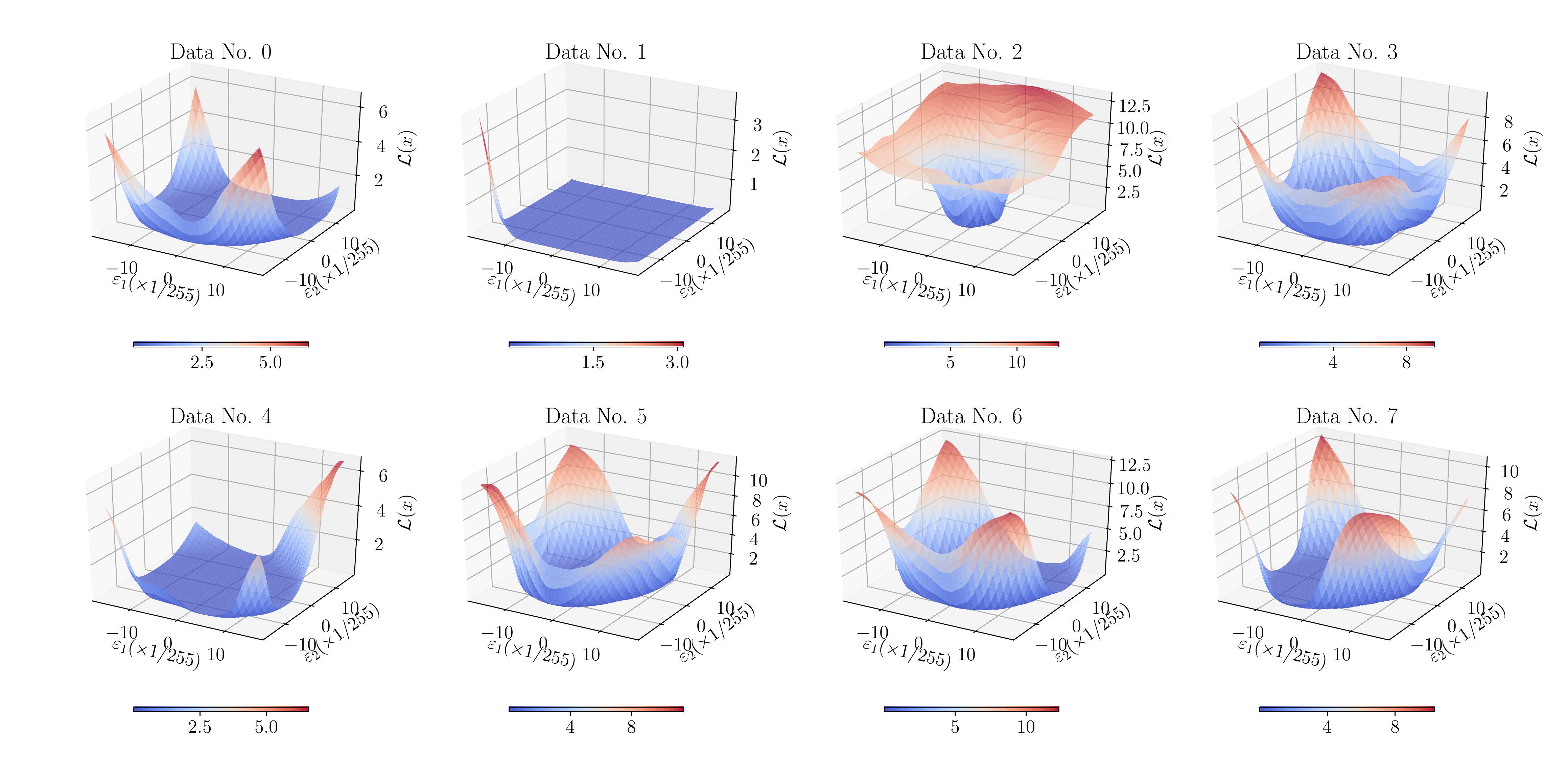}
    \caption{Loss Surface over input spaces of Baseline (standard training) on CIFAR10.}
    \label{NLoss}
\end{figure*}
\begin{figure*}
    \centering
    \includegraphics[width=\linewidth]{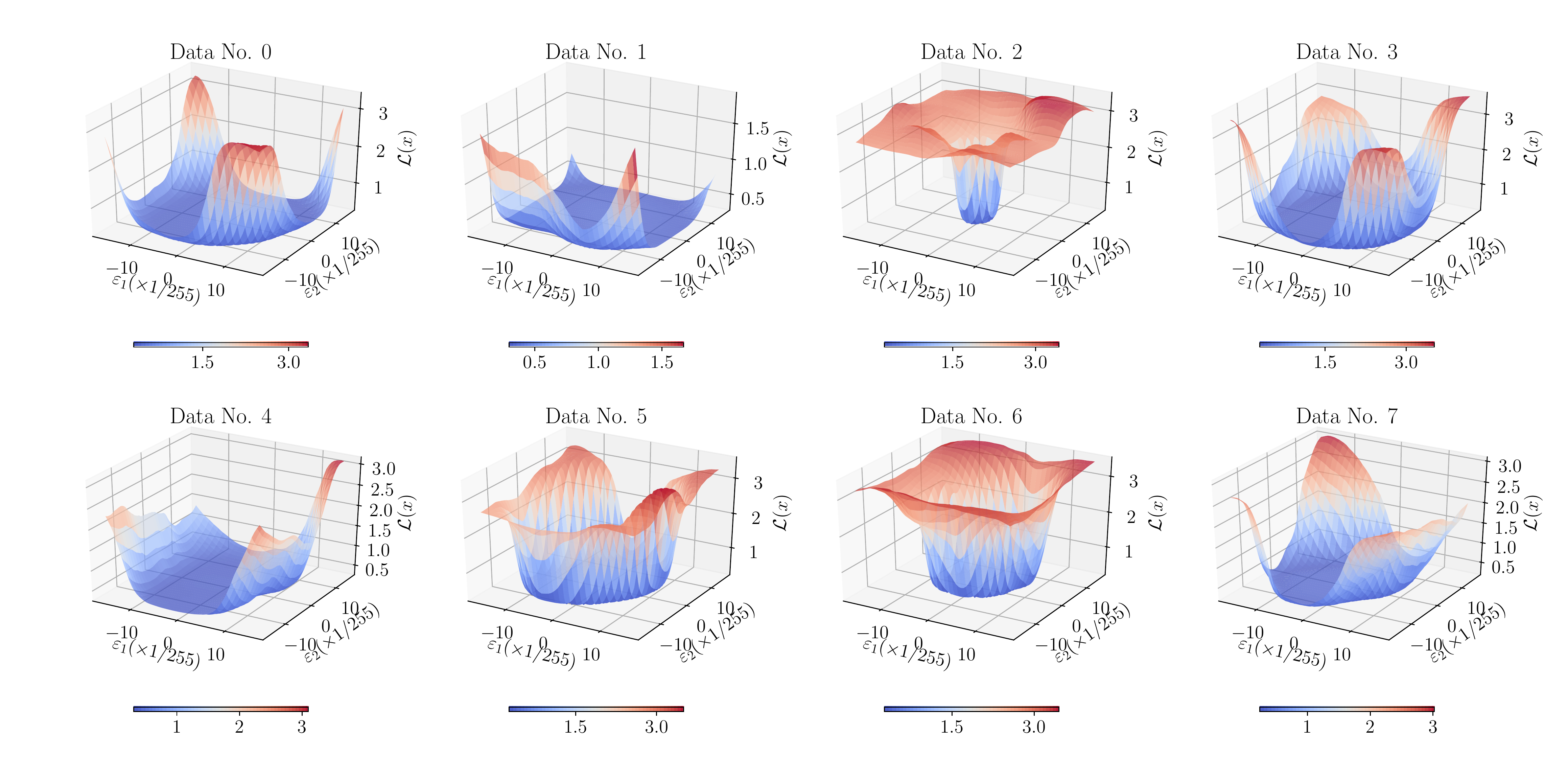}
    \caption{Loss Surface over input spaces of LSQ (standard training) on CIFAR10.}
    \label{NLSQLoss}
\end{figure*}
\begin{figure*}
    \centering
    \includegraphics[width=\linewidth]{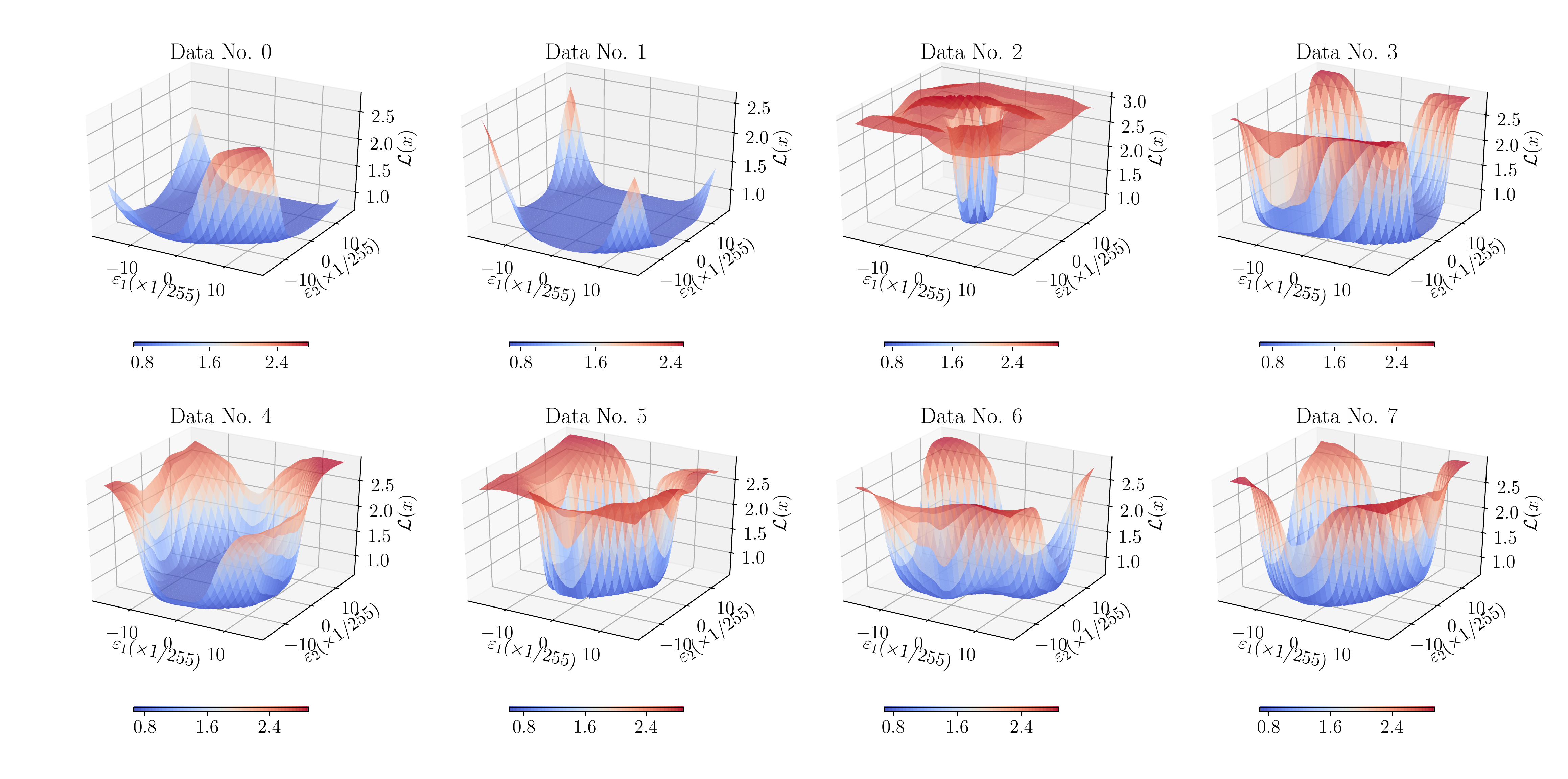}
    \caption{Loss Surface over input spaces of LSM (standard training) on CIFAR10.}
    \label{NLSMLoss}
\end{figure*}
\begin{figure*}
    \centering
    \includegraphics[width=\linewidth]{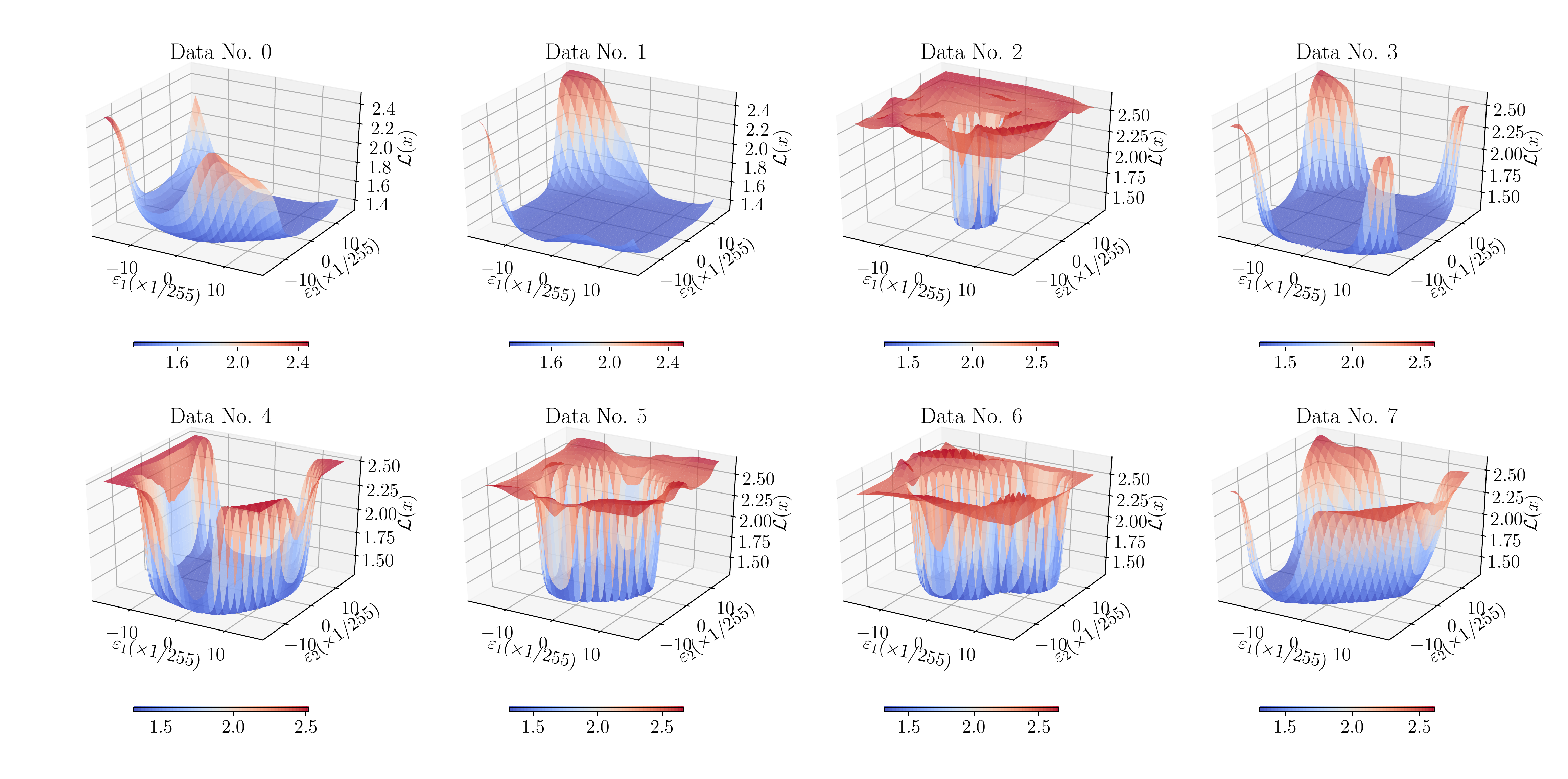}
    \caption{Loss Surface over input spaces of BLF (standard training) on CIFAR10.}
    \label{NBLFLoss}
\end{figure*}
\begin{figure*}
    \centering
    \includegraphics[width=\linewidth]{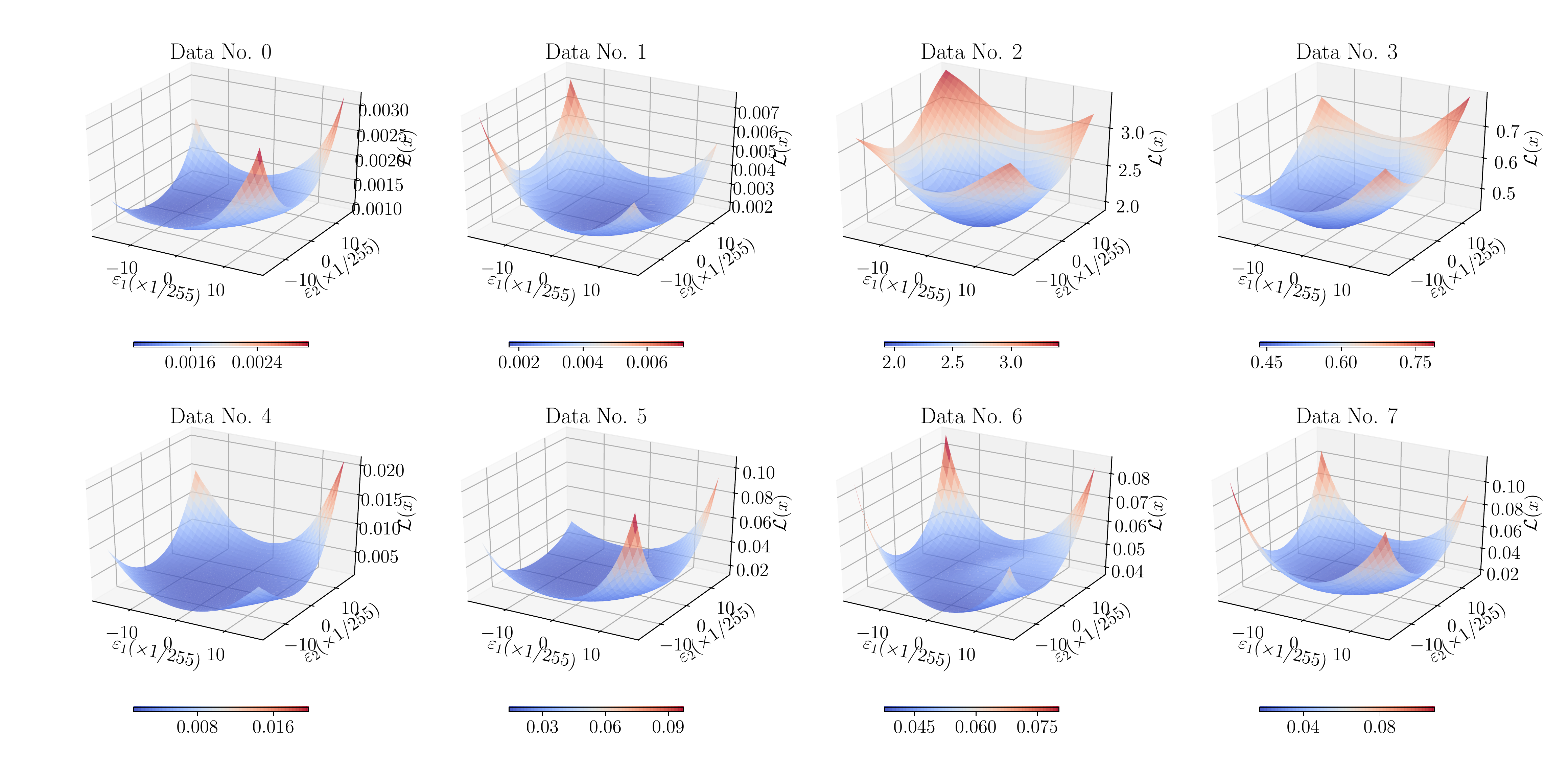}
    \caption{Loss Surface over input spaces of Baseline (adversarial training) on CIFAR10.}
    \label{PLoss}
\end{figure*}
\begin{figure*}
    \centering
    \includegraphics[width=\linewidth]{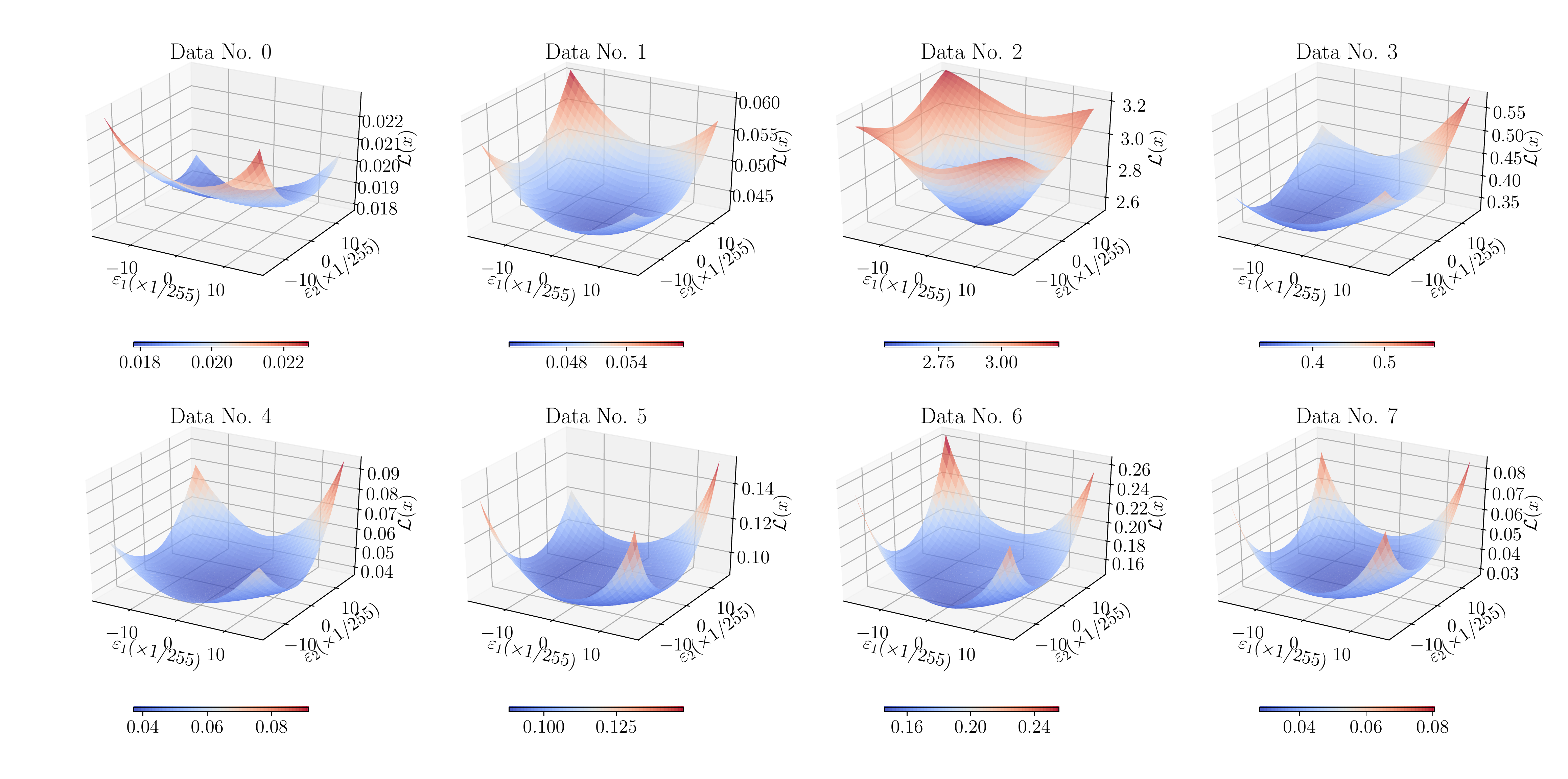}
    \caption{Loss Surface over input spaces of LSQ (adversarial training) on CIFAR10.}
    \label{PLSQLoss}
\end{figure*}
\begin{figure*}
    \centering
    \includegraphics[width=\linewidth]{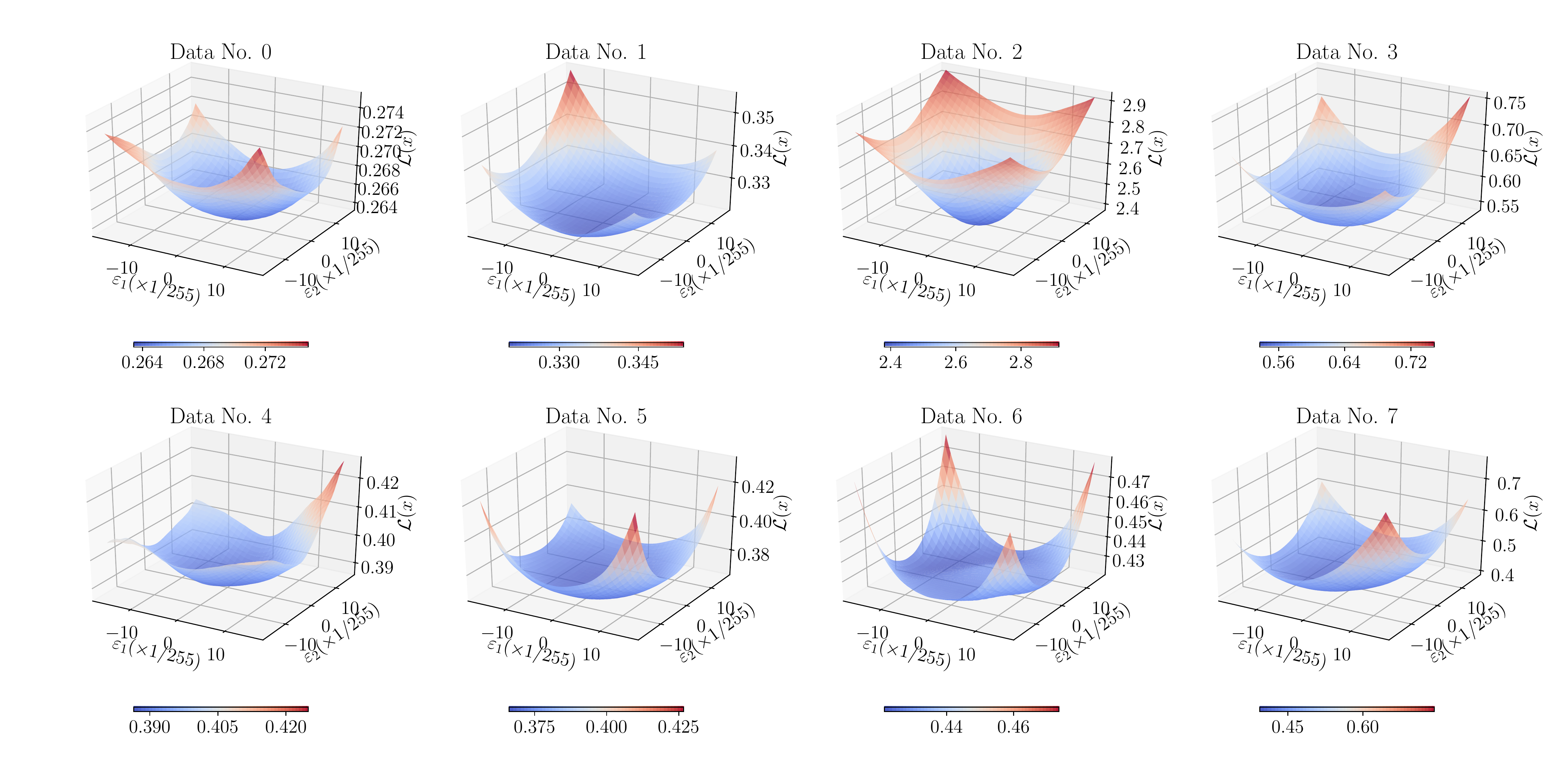}
    \caption{Loss Surface over input spaces of LSM (adversarial training) on CIFAR10.}
    \label{PLSMLoss}
\end{figure*}
\begin{figure*}
    \centering
    \includegraphics[width=\linewidth]{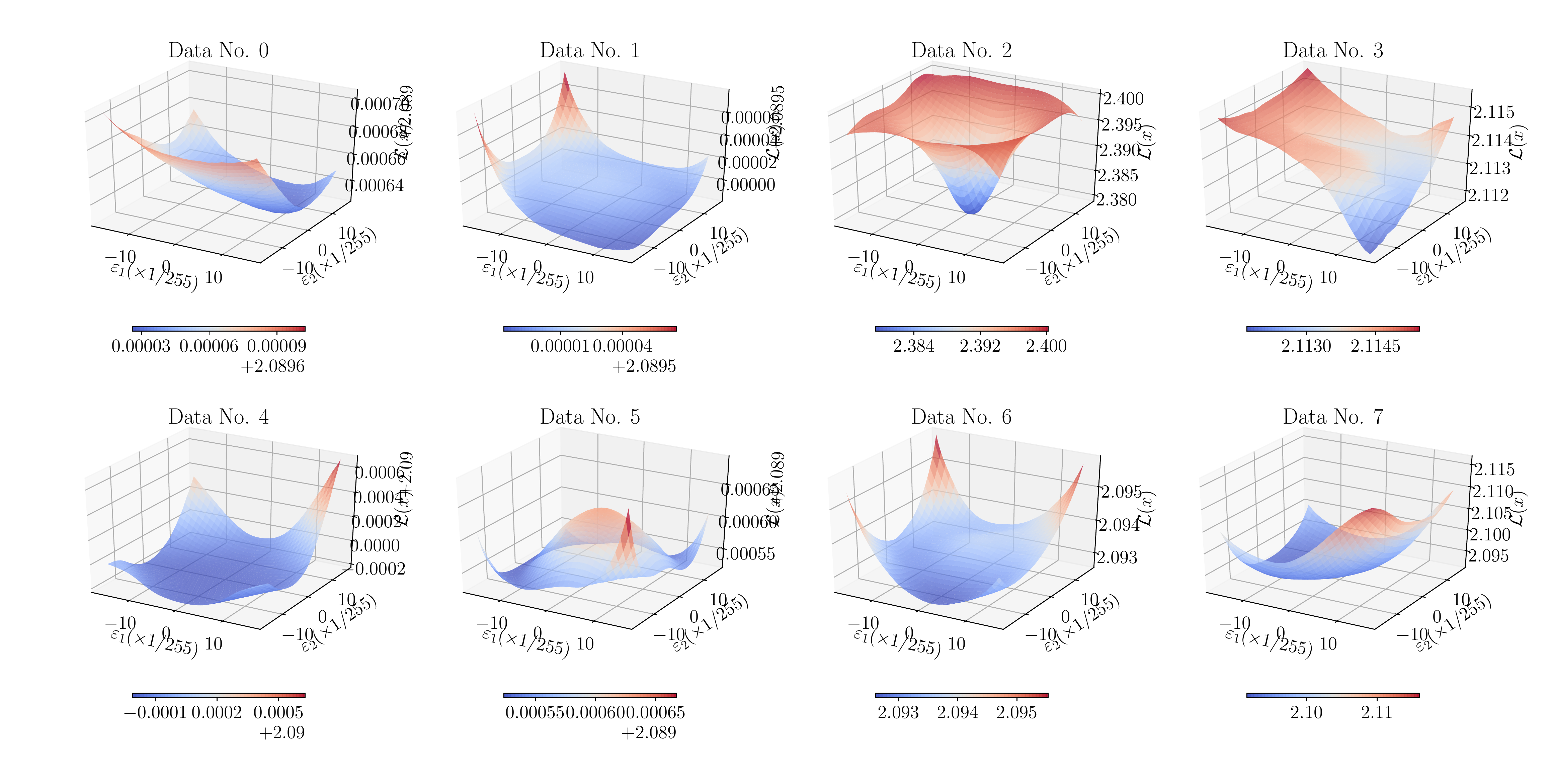}
    \caption{Loss Surface over input spaces of BLF (adversarial training) on CIFAR10.}
    \label{PBLFLoss}
\end{figure*}
\begin{figure*}
    \centering
    \includegraphics[width=\linewidth]{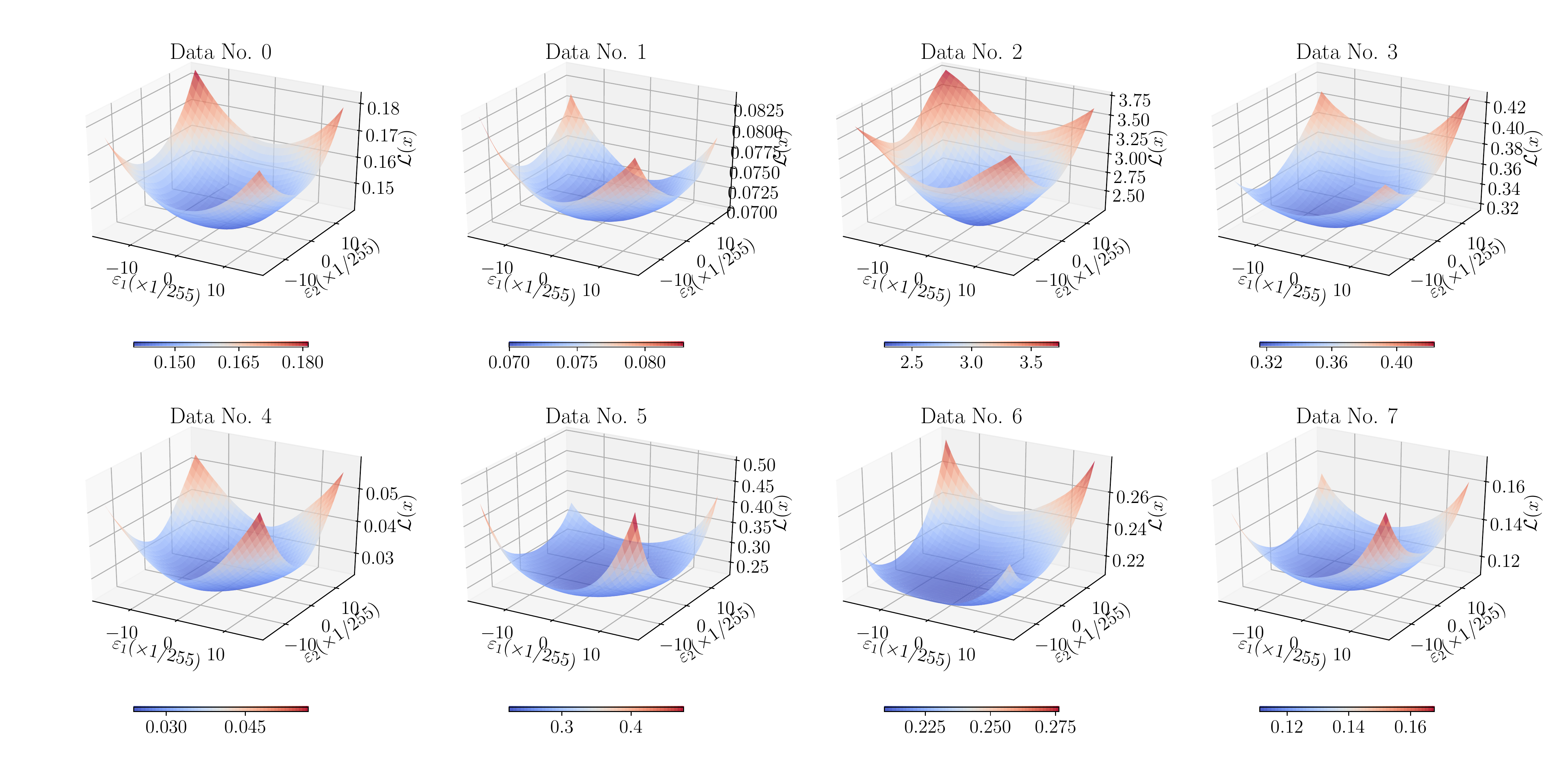}
    \caption{Loss Surface over input spaces of TRADES (adversarial training) on CIFAR10.}
    \label{PTRADESLoss}
\end{figure*}
\begin{table}[tb]
    \centering
    \caption{The difference between the maximum loss and minimum loss for random noise in the input loss surface experiment:
     $\max_{\varepsilon_1, \varepsilon_2} \mathcal{L}_{CE}(\bm{x}+\varepsilon_1\bm{v}_1+\varepsilon_2\bm{v}_2)-\min_{\varepsilon_1, \varepsilon_2} \mathcal{L}_{CE}(\bm{x}+\varepsilon_1\bm{v}_1+\varepsilon_2\bm{v}_2)$
     where $-16/255\leq\varepsilon_*\leq 16/255$. Results are averaged over eight data points.}
     \label{MMdiff}
    \begin{tabular}{ccccccc}\toprule
        &Baseline&LSQ&LSM&BLF&TRADES\\\midrule
        ST&9.38&2.89&2.21&1.28&N/A\\
        AT&0.269&0.160&0.168&0.00680&0.263\\
        \bottomrule
    \end{tabular}
\end{table}

\end{document}